\documentclass{article}

\usepackage{microtype}
\usepackage{graphicx}
\usepackage{subfigure}
\usepackage{booktabs} 

\usepackage{hyperref}

\usepackage{algorithm}
\usepackage{algorithmic}

\usepackage[preprint]{icml2026}

\usepackage{amsmath}
\usepackage{amssymb}
\usepackage{mathtools}
\usepackage{amsthm}
\usepackage{bbm}
\usepackage{makecell}

\usepackage[capitalize,noabbrev]{cleveref}

\usepackage{amsmath,amsfonts,bm}

\def\eqref#1{equation~\ref{#1}}

\def\1{\bm{1}}

\DeclareMathAlphabet{\mathsfit}{\encodingdefault}{\sfdefault}{m}{sl}
\SetMathAlphabet{\mathsfit}{bold}{\encodingdefault}{\sfdefault}{bx}{n}

\usepackage{multirow}
\usepackage{xspace}
\usepackage{colortbl}

\newcommand{\method}{MolField\xspace}

\theoremstyle{plain}
\newtheorem{theorem}{Theorem}[section]
\newtheorem{proposition}[theorem]{Proposition}
\newtheorem{lemma}[theorem]{Lemma}

\theoremstyle{definition}
\newtheorem{definition}[theorem]{Definition}

\theoremstyle{remark}

\usepackage[textsize=tiny]{todonotes}

\begin{document}

\twocolumn[
    \icmltitle{Molecular Representations in Implicit Functional Space via Hyper-Networks}



    \icmlsetsymbol{equal}{*}

    \begin{icmlauthorlist}
        \icmlauthor{Zehong Wang}{nd,equal}
        \icmlauthor{Xiaolong Han}{surrey,equal}
        \icmlauthor{Qi Yang}{in,equal}
        \icmlauthor{Xiangru Tang}{yale}
        \icmlauthor{Fang Wu}{stf}
        \icmlauthor{Xiaoguang Guo}{uconn}
        \\
        \icmlauthor{Weixiang Sun}{nd}
        \icmlauthor{Tianyi Ma}{nd}
        \icmlauthor{Pietro Li\`o}{cam}
        \icmlauthor{Le Cong}{stf}
        \icmlauthor{Sheng Wang}{uw}
        \icmlauthor{Chuxu Zhang}{uconn}
        \icmlauthor{Yanfang Ye}{nd}
    \end{icmlauthorlist}

    \icmlaffiliation{nd}{University of Notre Dame}
    \icmlaffiliation{uconn}{University of Connecticut}
    \icmlaffiliation{in}{Independent Researcher Visiting Zhang's Lab in UConn}
    \icmlaffiliation{surrey}{University of Surrey}
    \icmlaffiliation{stf}{Stanford University}
    \icmlaffiliation{cam}{University of Cambridge}
    \icmlaffiliation{yale}{Yale University}
    \icmlaffiliation{uw}{University of Washington}

    \icmlcorrespondingauthor{Zehong Wang}{zwang43@nd.edu}
    \icmlcorrespondingauthor{Yanfang Ye}{yye7@nd.edu}

    \icmlkeywords{Machine Learning, ICML}

    \vskip 0.3in
]



\printAffiliationsAndNotice{}  

\begin{abstract}

    Molecular representations fundamentally shape how machine learning systems reason about molecular structure and physical properties.
    Most existing approaches adopt a discrete pipeline: molecules are encoded as sequences, graphs, or point clouds, mapped to fixed-dimensional embeddings, and then used for task-specific prediction.
    This paradigm treats molecules as discrete objects, despite their intrinsically continuous and field-like physical nature.
    We argue that molecular learning can instead be formulated as learning in function space.
    Specifically, we model each molecule as a continuous function over three-dimensional (3D) space and treat this molecular field as the primary object of representation.
    From this perspective, conventional molecular representations arise as particular sampling schemes of an underlying continuous object.
    We instantiate this formulation with \method, a hyper-network-based framework that learns distributions over molecular fields.
    To ensure physical consistency, these functions are defined over canonicalized coordinates, yielding invariance to global SE(3) transformations.
    To enable learning directly over functions, we introduce a structured weight tokenization and train a sequence-based hyper-network to model a shared prior over molecular fields.
    We evaluate \method on molecular dynamics and property prediction.
    Our results show that treating molecules as continuous functions fundamentally changes how molecular representations generalize across tasks and yields downstream behavior that is stable to how molecules are discretized or queried.
\end{abstract}

\begin{center}
    \textit{``The limits of language define the limits of our world.''}
\end{center}
\begin{flushright}
    \vspace{-0.3cm}
    \textit{--- Ludwig Wittgenstein}
\end{flushright}

\section{Introduction}

Representations determine how scientific objects are formalized, compared, and manipulated, and thereby shape what can be systematically learned about them \citep{frigg2016scientific,piantadosi2021computational,wallace2024learning}.
A physical system may be described through symbolic rules, discrete structures, or continuous mathematical objects, with each choice highlighting different aspects of the same underlying reality \citep{franklin2017discrete}.
In this sense, representations play a role analogous to language: they define not only how knowledge is expressed, but also what kinds of reasoning and generalization are possible \citep{frigg2016scientific,stanley2015greatness}.
Consequently, representation design does not merely affect modeling convenience, but fundamentally constrains how scientific phenomena can be learned and transferred across settings \citep{morgan1999models}.

\begin{figure}[!t]
    \centering
    \includegraphics[width=\linewidth]{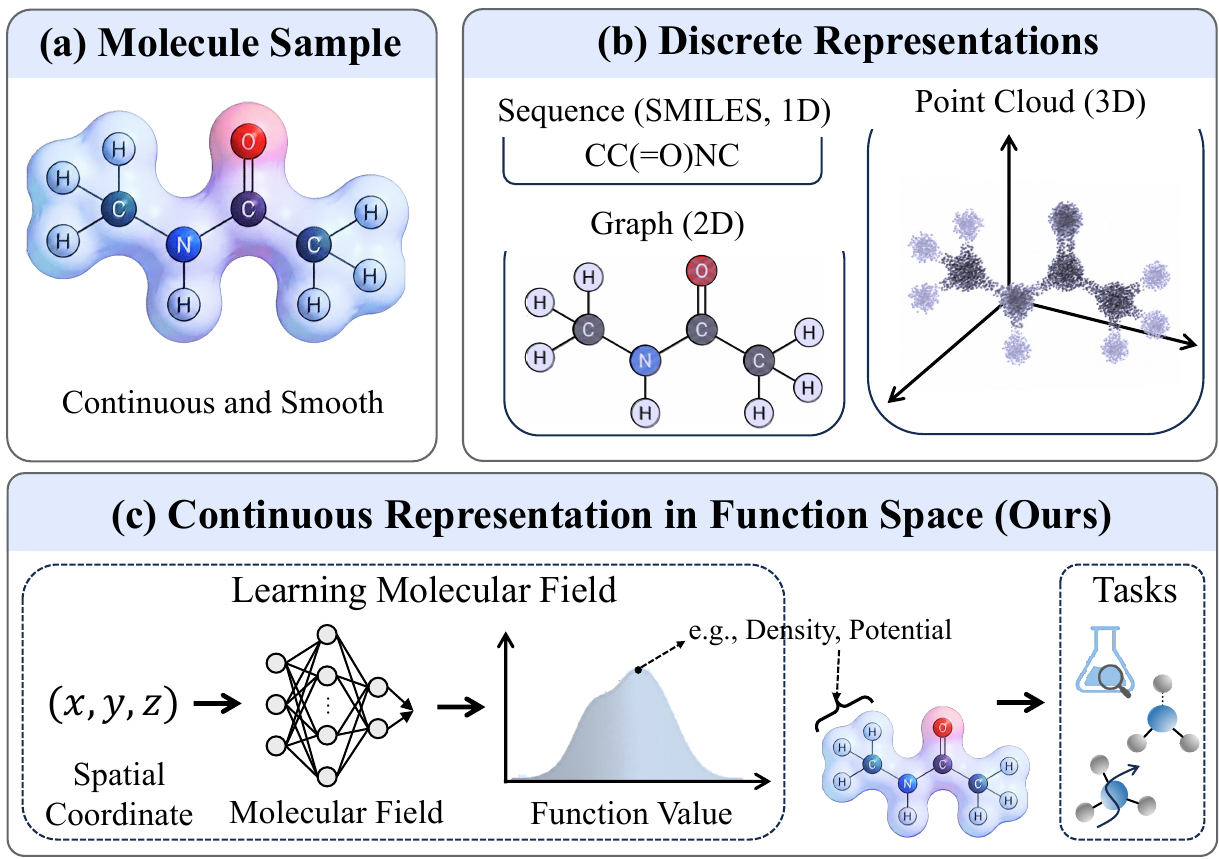}
    \caption{
        \textbf{Molecular representations: discrete versus continuous function-space modeling.}
        (a) A molecule exhibits inherently continuous and smooth spatial structure.
        (b) Conventional molecular representations discretize molecules into sequences, graphs, or point clouds, which only partially capture this continuity.
        (c) In contrast, we model molecules directly in function space by learning a continuous molecular field that maps 3D coordinates to physical quantities (e.g., density or potential), enabling a unified representation for diverse downstream tasks.
    }
    \label{fig:intro}
\end{figure}

Molecular systems present a particularly challenging instance of this representational problem.
A molecule is a physical entity embedded in three-dimensional (3D) space, whose identity is determined not only by discrete atom and bond types, but also by geometry, shape, and their evolution under physical processes \citep{ashcroft2006molecule,crippen1988distance}.
Many fundamental molecular quantities—such as fields, surfaces, and energies—are therefore most naturally described as continuous functions over space \citep{vanommeslaeghe2014molecular}.
Despite this intrinsic continuity, prevailing molecular learning pipelines treat molecules as discrete structures: sequences, graphs, or point clouds are first constructed, mapped to fixed-dimensional embeddings, and then consumed by task-specific predictors (Figure~\ref{fig:intro}).
These discretizations act as particular sampling schemes of the underlying molecular object, but inevitably discard fine-grained continuity and impose representation-specific inductive biases.
As a result, learned representations become tightly coupled to individual tasks and discretization choices, limiting their ability to generalize across tasks \citep{schutt2017schnet,gasteiger2020directional}.

To this end, we formulate molecular learning directly in function space.
Specifically, we represent each molecule as a continuous function defined over 3D space \citep{sitzmann2020implicit,essakine2024we}, whose values encode physically meaningful molecular quantities.
Under this formulation, the molecular function itself becomes the primary object of representation, while downstream tasks correspond to different query operators applied to a shared underlying function, rather than to separate task-specific embeddings.

Formulating molecular learning in function space fundamentally changes the object of representation and, consequently, the learning paradigm.
This shift introduces three core challenges.
\textit{(1) Defining molecules as functions.}
Replacing discrete structures with continuous functions requires molecular representations to be well-defined and physically consistent in continuous space.
While molecular identity is invariant to global translations and rotations \citep{batzner20223}, continuous functions are inherently sensitive to coordinate choices \citep{ran2023neurar}.
Without explicitly accounting for rigid-body symmetries, physically identical molecular configurations may correspond to different functions, breaking representational consistency.
\textit{(2) Learning molecular functions.}
Standard molecular learning pipelines are designed to map discrete structures to fixed-dimensional embeddings.
By contrast, molecular functions are infinite-dimensional objects without an intrinsic discrete interface for comparison, interpolation, or generalization.
Naively discretizing them reintroduces resolution dependence and distorts functional semantics.
A central challenge is therefore to design a structured, learnable interface that exposes continuous molecular functions to modern learning architectures while preserving their compositional structure.
\textit{(3) Using molecular functions in downstream tasks.}
If the function itself is the representation, downstream learning can no longer rely on task-specific embeddings.
Instead, different tasks must be solved by querying a shared underlying function.
This requires a unified paradigm in which molecular functions serve as task-agnostic objects of learning.

To address these challenges, we propose \method, a unified framework that operationalizes molecular learning directly in function space.
\method represents each molecule as a continuous molecular field using a \emph{Canonical Implicit Neural Representation (C-INR)}, which defines molecular functions over canonicalized coordinates invariant to SE(3) transformations, ensuring that the function depends only on intrinsic molecular geometry (Challenge~1).
To expose molecular functions to modern sequence models, we introduce \emph{Structured Weight Tokenization (SWT)}, which converts C-INR parameters into semantically organized tokens while preserving the compositional structure of the underlying neural function (Challenge~2).
Building on this interface, \method employs a \emph{Function Space Hyper-Network (FSHN)} \citep{ha2017hypernetworks} that learns distributions over molecular functions by directly generating C-INR parameters, enabling function-level generalization and amortized instantiation instead of per-instance optimization (Challenge~3).
Extensive experiments on molecular dynamics and property prediction demonstrate that learning distributions in function space yields generalizable representations.

This paper makes the following contributions:
\begin{itemize}
    \item We formulate molecular representation as learning in function space, treating each molecule as a continuous, SE(3)-invariant function over 3D space and viewing conventional graphs, point clouds, and sequences as discretization schemes of this underlying object.
    \item We propose \method, a unified framework that operationalizes this formulation using canonical implicit neural representations, structured weight tokenization, and a function-space hyper-network to learn distributions over molecular functions end-to-end.
    \item We validate the proposed paradigm on molecular dynamics and property prediction, showing that function-space representations yield downstream behavior that is stable to how molecules are discretized or queried.
\end{itemize}
\begin{figure*}[!t]
    \centering
    \includegraphics[width=\linewidth]{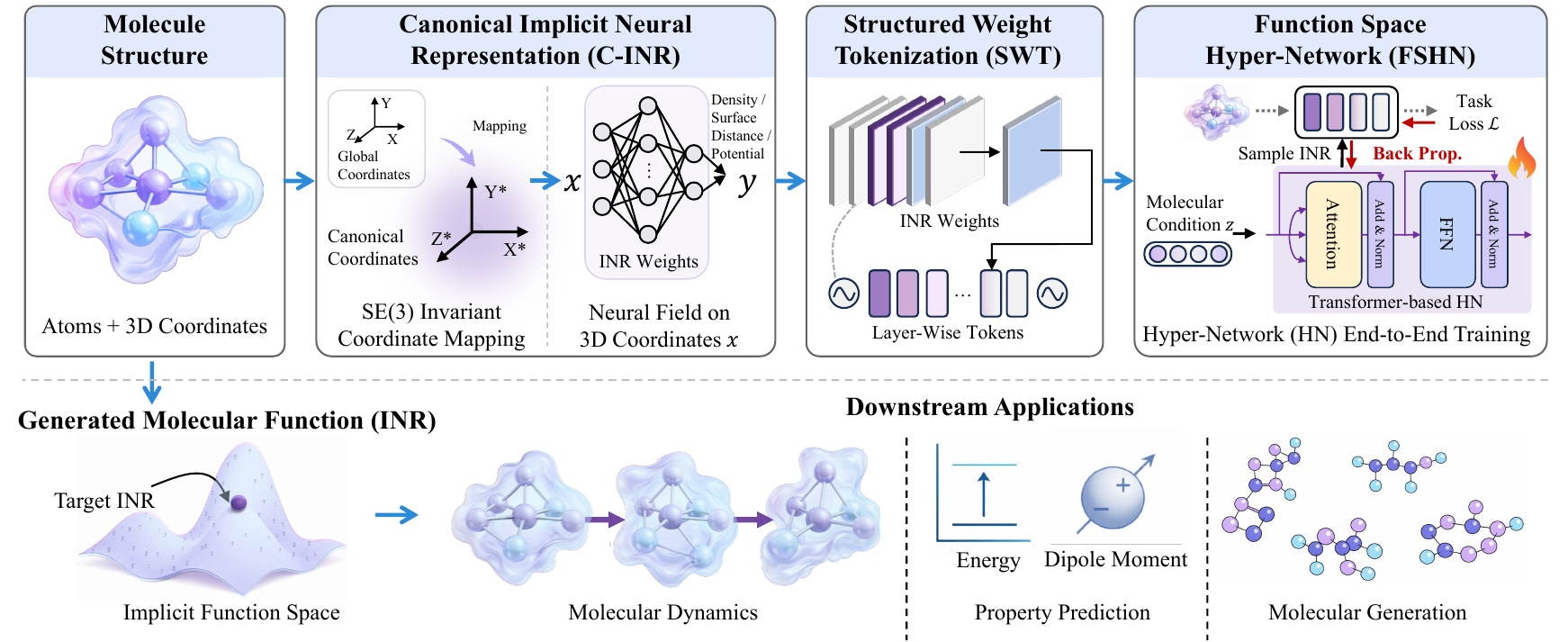}
    \caption{
    \textbf{Overview of \method.}
    \textit{Top}: A molecule is mapped into a canonical coordinate system to construct an implicit neural representation (INR) that is invariant to SE(3) transformations. The INR weights are tokenized in a structured, layer-wise manner to enable transformer-based processing, while the model is trained as a hyper-network that learns to generate task-conditioned INRs. \textit{Bottom}: The resulting INR supports diverse molecular learning tasks, including dynamic surface modeling and property prediction, and can be naturally extended to molecular generation (Appendix \ref{sec:generation}).
    }
    \label{fig:framework}
\end{figure*}

\section{Preliminaries}

\begin{definition}[Molecular Functions]
    We model molecules as continuous functions via INRs defined over 3D Euclidean space.
    Given a molecular configuration described by atomic coordinates $X = \{\mathbf{x}_i \in \mathbb{R}^3\}_{i=1}^N$ and atomic attributes $A = \{a_i\}_{i=1}^N$, a molecular function is defined as
    \begin{equation}
        f_M : \mathbb{R}^3 \rightarrow \mathbb{R}^C,
    \end{equation}
    where $f_M(\mathbf{x})$ encodes molecular quantities at spatial location $\mathbf{x}$, such as density, potential, or other field-valued properties \citep{sun2023dsr,kirchmeyer2024score}.
    In practice, this function is implemented as a MLP with parameters $\theta$, which serves as a compact and resolution-independent encoding of the underlying function.
\end{definition}

\begin{definition}[SE(3) Symmetry]
    Molecular functions admit intrinsic symmetries of 3D space.
    Specifically, for any rigid-body transformation $g=(R,\mathbf{t}) \in \mathrm{SE}(3)$ with $R \in \mathrm{SO}(3)$ and $\mathbf{t} \in \mathbb{R}^3$, the function transforms as
    \begin{equation}
        f_M(\mathbf{x}) \mapsto f_M(R^{-1}(\mathbf{x}-\mathbf{t})).
    \end{equation}
    Since such transformations preserve molecular identity, molecular representations must be invariant to the action of $\mathrm{SE}(3)$.
    We restrict our consideration to $\mathrm{SE}(3)$ rather than the full Euclidean group $\mathrm{E}(3)$, as reflections alter molecular chirality and do not preserve molecular identity.
\end{definition}

\paragraph{Problem Definition.}
Given a molecule specified by atomic coordinates $X$ and atomic attributes $A$, our objective is to learn a continuous molecular function $f_M$ parameterized by an implicit neural representation.
The learned function must provide a resolution-independent encoding of molecular structure, remain invariant under global rigid-body transformations in $\mathrm{SE}(3)$, and support diverse downstream learning objectives.
We aim to enable learning and inference to operate directly in the molecular function space.

\section{Methodology}

We propose \method, a framework for learning and generating molecular representations directly in implicit functional space (Figure~\ref{fig:framework}).
At a high level, \method decomposes the problem into three stages: constructing a symmetry-consistent implicit representation for individual molecules, exposing this representation through a structured tokenization, and learning a conditional hyper-network over such representations.

\subsection{Canonical Implicit Neural Representations (C-INR)}
\label{sec:cinr}

We define a C-INR as a coordinate-based neural function that models a molecule as a continuous field in $\mathbb{R}^3$ while ensuring invariance to global rigid-body transformations.
Given atomic coordinates $X=\{\mathbf{x}_i\}_{i=1}^{N}$, a C-INR parameterizes a molecular function
$f_\theta : \mathbb{R}^3 \rightarrow \mathbb{R}^C$,
whose outputs represent physically meaningful quantities such as signed distance or density.
The key property of C-INR is that the function is evaluated in a molecule-specific canonical coordinate system, such that the parameters $\theta$ depend only on intrinsic molecular geometry.

\paragraph{Canonical Coordinate Mapping.}
We remove pose ambiguity by expressing spatial queries in a canonical reference frame.
Specifically, let $\bar{\mathbf{x}}=\frac{1}{N}\sum_i \mathbf{x}_i$ denote the molecular centroid and let $Q(X)\in SO(3)$ be a rotation-equivariant canonical frame derived from molecular geometry.
We define the coordinate transformation
\begin{equation}
    c(\mathbf{x};X)=Q(X)^{\top}(\mathbf{x}-\bar{\mathbf{x}}),
\end{equation}
and evaluate the molecular field as
\begin{equation}
    f_{\theta}^{\mathrm{C}}(\mathbf{x};X)=f_{\theta}\!\left(c(\mathbf{x};X)\right).
\end{equation}
This formulation standardizes the function domain across different global poses without aligning distinct molecules into a shared global frame.

\paragraph{Canonical Frame Construction.}
The frame $Q(X)$ is constructed from rotation-equivariant vector features extracted from molecular geometry.
Let $\Phi(X)=\{\mathbf{v}_i(X)\}_{i=1}^{N}$ be a mapping that satisfies $\Phi(RX)=\{R\mathbf{v}_i(X)\}_{i=1}^{N}$ for any $R\in SO(3)$ (see Appendix~\ref{sec:encoder} for details).
We form two global axes via invariant aggregation,
\begin{equation}
    \mathbf{u}_k(X)=\sum_{i=1}^{N}\alpha_i^{(k)}(X)\mathbf{v}_i(X), \quad k\in\{1,2\},
\end{equation}
where $\alpha_i^{(k)}(X)$ are rotation-invariant scalars.
An orthonormal basis is then obtained through a modified Gram--Schmidt procedure,
\begin{equation}
    Q(X)=[\mathbf{q}_1(X),\mathbf{q}_2(X),\mathbf{q}_3(X)]\in SO(3),
\end{equation}
with $\mathbf{q}_1=\mathbf{u}_1/\|\mathbf{u}_1\|$, $\mathbf{q}_2$ the normalized component of $\mathbf{u}_2$ orthogonal to $\mathbf{q}_1$, and $\mathbf{q}_3=\mathbf{q}_1\times\mathbf{q}_2$.

\paragraph{SE(3) Invariance.}
C-INR achieves invariance to global rigid-body transformations by construction, through its canonical coordinate interface.
This property follows from the rotation equivariance of the canonical frame and the removal of translations by centering.

\begin{proposition}[Frame equivariance]
    \label{prop:frame_equi}
    For any $R \in SO(3)$, the canonical frame satisfies
    \(
    Q(RX) = R\,Q(X).
    \)
\end{proposition}

\begin{lemma}[SE(3)-invariant coordinate mapping]
    \label{lemma:se3}
    For any rigid-body transformation $g=(R,\mathbf{t}) \in SE(3)$,
    \begin{equation}
        c(R\mathbf{x}+\mathbf{t}; RX+\mathbf{t}) = c(\mathbf{x};X).
    \end{equation}
\end{lemma}

Lemma~\ref{lemma:se3} follows directly from Proposition~\ref{prop:frame_equi} and the definition of the centroid-based coordinate transform.
Consequently, for any two rigid-body equivalent molecular configurations, all spatial queries are mapped to identical canonical coordinates.
Since the molecular field is evaluated as
\begin{equation}
    f_{\theta}^{\mathrm{C}}(\mathbf{x};X)=f_{\theta}\!\left(c(\mathbf{x};X)\right),
\end{equation}
the induced function is invariant under SE(3), and the representation parameters $\theta$ depend only on intrinsic molecular geometry.
The full proofs are provided in Appendix~\ref{sec:proof_symmetry}.

\subsection{Structured Weight Tokenization (SWT)}
\label{sec:tokenization}

C-INRs represent molecules as continuous functions whose parameters are invariant to global SE(3) transformations.
However, these parameters are high-dimensional tensors with rich internal structure, and treating them as flat vectors discards the architectural organization of the underlying neural function.
To enable learning directly over molecular functions, we introduce a structured tokenizer that exposes C-INR weights in a form amenable to sequence models.

A C-INR is parameterized by a collection of layer-wise tensors $\theta=\{\theta^{(l)}\}_{l=1}^{L}$, where each layer implements a distinct functional transformation.
SWT converts this structured parameter set into an ordered token sequence
\begin{equation}
    \theta \;\longrightarrow\; \tau_{1:T}, \quad \tau_t \in \mathbb{R}^{d_{\mathrm{model}}},
\end{equation}
where each token corresponds to a localized parameter block associated with a specific layer and parameter type (e.g., weight or bias).
The token ordering follows the architectural hierarchy of the C-INR, thereby preserving the causal structure of neural computation.

Since parameter tokens are non-exchangeable and carry explicit architectural semantics, we encode their structural roles using learnable positional embeddings.
Each token is indexed by a tuple $(l,r)$ indicating its layer $l$ and parameter role $r$, and its representation is defined as
\begin{equation}
    \tilde{\tau}_t = \tau_t + \mathbf{e}_{\mathrm{layer}}(l) + \mathbf{e}_{\mathrm{role}}(r),
\end{equation}
where $\mathbf{e}_{\mathrm{layer}}$ and $\mathbf{e}_{\mathrm{role}}$ are trainable embeddings.
This encoding injects inductive bias about network depth and parameter functionality, allowing downstream models to reason over molecular functions while respecting the compositional structure of the implicit representation.

\subsection{Function Space Hyper-Network (FSHN)}
\label{sec:hypernetwork}

Under the C-INR formulation, each molecule is represented by a neural function $f_\theta$ whose parameters uniquely specify the corresponding molecular field.
Rather than predicting task-specific outputs, molecular learning can therefore be viewed as learning distributions over such functions.
The intuitive idea is to fit INRs independently for each molecule and subsequently learn over the resulting parameters, but this two-stage procedure introduces optimization-dependent artifacts and prevents consistent reasoning in function space.
We address this limitation by learning a conditional hyper-network over C-INRs in an end-to-end manner.

Specifically, we introduce a function-space hyper-network $G_\phi$ that maps a conditioning variable $z$ to C-INR parameters
\begin{equation}
    \theta = G_\phi(z).
\end{equation}
This defines a conditional distribution over molecular functions via their parameters, i.e., $p(f \mid z) \equiv p(\theta \mid z)$, enabling learning directly at the level of function families.

Building on SWT (Section~\ref{sec:tokenization}), the C-INR parameters are represented as a structured token sequence $\tau_{1:T}$.
The hyper-network models their joint distribution autoregressively,
\begin{equation}
    p(\theta \mid z) = \prod_{t=1}^{T} p(\tau_t \mid \tau_{<t}, z),
\end{equation}
which enforces global coherence while respecting the hierarchical structure of the underlying neural representation.

During training, the hyper-network generates C-INR parameters $\theta = G_\phi(z)$, which instantiate a molecular function $f_\theta$.
Task-specific predictions are obtained by querying this function through $\mathcal{T}(\cdot)$ and optimized using the task loss
\begin{equation}
    \mathcal{L}_{\mathrm{task}}(\phi)
    =
    \mathbb{E}_{(z, y) \sim \mathcal{D}}
    \big[
    \ell(\mathcal{T}(f(\cdot; G_{\phi}(z))), y)
    \big].
\end{equation}
Gradients propagate through the downstream task loss from INR to update $\phi$, yielding a single end-to-end training procedure.
At inference time, molecular functions are generated in a single forward pass conditioned on $z$ and queried according to the downstream task, without additional optimization.

\subsection{Training Procedure}

\paragraph{End-to-End Training Pipeline.}
Given a conditioning input $z$, the hyper-network produces C-INR parameters $\theta = G_\phi(z)$.
These parameters specify a continuous molecular function
\(
f_\theta : \mathbb{R}^3 \rightarrow \mathbb{R}^Z
\)
evaluated on canonicalized coordinates.
Depending on the downstream task, the generated function is either instantiated and queried at spatial locations, or its structured parameters are directly read out as a representation.
In both cases, task losses supervise the outputs and gradients are backpropagated through the C-INR and the hyper-network to update $\phi$.
This procedure allows MolField to learn distributions over molecular functions in a unified and task-agnostic manner.

\paragraph{Conditioning Variables.}
The conditioning variable $z$ specifies which molecular function should be generated.
(1) For \textit{molecular dynamics}, $z$ is formed by concatenating a learned embedding of the molecule with the target time step $t$.
The generated C-INR therefore represents the molecular field of the same molecule at time $t$.
(2) For \textit{property prediction}, $z$ is given by the molecular embedding of the static structure.
The generated parameters serve as a function-space representation of the molecule for downstream regression.
(3) For \textit{molecular generation} (Appendix \ref{sec:generation}), we consider the unconditional setting and sample $z$ from a standard Gaussian prior, $z \sim \mathcal{N}(0, I)$.
The hyper-network thus generates molecular functions solely from the learned distribution, without additional semantic constraints.

\paragraph{Training Objectives.}
The C-INR architecture and canonical coordinate system are shared across tasks.
Tasks differ only in the physical semantics assigned to the function outputs and in the supervision signals used during training.
\textit{(1) Molecular Dynamics.}
The generated C-INR is instantiated as a signed distance function (SDF) \citep{sun2023dsr}
\(f_\theta : \mathbb{R}^3 \rightarrow \mathbb{R}\),
which maps spatial coordinates to distances from the molecular surface.
Given sampled spatial locations $\mathbf{x}$, we supervise the field using
\(
\mathcal{L}_{\text{MD}}
=
\mathcal{L}_{\text{SDF}}
+
\lambda \mathcal{L}_{\text{eikonal}},
\)
where $\mathcal{L}_{\text{SDF}}$ enforces accurate reconstruction of surface distances and
$\mathcal{L}_{\text{eikonal}} = \mathbb{E}_{\mathbf{x}}(\|\nabla_{\mathbf{x}} f_\theta(\mathbf{x})\| - 1)^2$
regularizes gradient norms to encourage valid geometric fields.
\textit{(2) Property Prediction.}
For property prediction, we do not explicitly instantiate the neural field.
Instead, the structured token sequence corresponding to the generated C-INR parameters is directly aggregated into a fixed-dimensional embedding.
This embedding is supervised by multiple physical property targets using a regression objective
$\mathcal{L}_{\text{prop}} = \|\hat{\mathbf{y}} - \mathbf{y}_{gt}\|_1$,
where $\hat{\mathbf{y}}$ denotes the predicted properties.
\textit{(3) Molecular Generation} (Appendix \ref{sec:generation}).
For molecular generation, the generated C-INR is instantiated as a continuous density field
\(
f_\theta : \mathbb{R}^3 \rightarrow \mathbb{R}^d,
\)
whose values represent the probability of the  existence of certain atoms at spatial locations.
We sample spatial points $\mathbf{x}$ from the 3D domain and supervise the predicted density against the ground-truth distribution $\rho_{gt}(\mathbf{x})$ using
$\mathcal{L}_{\text{gen}} = \mathbb{E}_{\mathbf{x}} \| f_\theta(\mathbf{x}) - \rho_{gt}(\mathbf{x}) \|^2$.

\section{Main Results}

\begin{table*}[!t]
    \centering
    \caption{
        \textbf{Evaluation of molecular dynamics.}
        We report volumetric IoU (IoU ↑), chamfer distance (CD ↓), and normal consistency (NC ↑), where the best performance is \textbf{boldface}.
        \method consistently achieves the best overall performance across datasets.
    }
    \label{tab:md_main}
    \resizebox{\linewidth}{!}{
        \begin{tabular}{l | ccc | ccc | ccc | ccc}
            \toprule
                      & \multicolumn{3}{c|}{NDF} & \multicolumn{3}{c|}{DSR} & \multicolumn{3}{c|}{CanFields} & \multicolumn{3}{c}{\method (Ours)}                                                                                                                                              \\ \cmidrule(lr){2-4} \cmidrule(lr){5-7} \cmidrule(lr){8-10} \cmidrule(lr){11-13}
            Dataset   & IoU$\uparrow$            & CD$\downarrow$           & NC$\uparrow$                   & IoU$\uparrow$                      & CD$\downarrow$  & NC$\uparrow$ & IoU$\uparrow$   & CD$\downarrow$  & NC$\uparrow$    & IoU$\uparrow$   & CD$\downarrow$  & NC$\uparrow$    \\ \midrule
            DIMS      & 0.9104                   & 0.0006                   & 0.9512                         & 0.9316                             & \textbf{0.0003} & 0.9671       & \textbf{0.9391} & 0.0005          & \textbf{0.9719} & 0.9353          & 0.0006          & 0.9645          \\
            YiiP      & 0.8120                   & 0.0006                   & 0.8515                         & 0.8425                             & \textbf{0.0003} & 0.8737       & 0.8404          & 0.0004          & 0.8354          & \textbf{0.9270} & 0.0006          & \textbf{0.9565} \\
            NahA      & 0.7968                   & 0.0008                   & 0.8124                         & 0.8265                             & 0.0005          & 0.8372       & 0.8362          & 0.0004          & 0.8027          & \textbf{0.8844} & \textbf{0.0003} & \textbf{0.8949} \\
            I-FABP    & 0.8729                   & 0.0008                   & 0.9196                         & 0.8907                             & 0.0005          & 0.9382       & 0.9135          & 0.0006          & 0.9270          & \textbf{0.9392} & \textbf{0.0002} & \textbf{0.9721} \\
            FRODA     & 0.8087                   & 0.0007                   & 0.8893                         & 0.8318                             & 0.0004          & 0.9098       & 0.8283          & \textbf{0.0003} & 0.8895          & \textbf{0.8687} & 0.0006          & \textbf{0.9414} \\
            adk\_equi & 0.7215                   & 0.0025                   & 0.7620                         & 0.7526                             & 0.0020          & 0.7914       & 0.7491          & 0.0019          & 0.7977          & \textbf{0.9131} & \textbf{0.0004} & \textbf{0.9522} \\ \midrule
            Average   & 0.8204                   & 0.0010                   & 0.8643                         & 0.8460                             & 0.0007          & 0.8862       & 0.8511          & 0.0007          & 0.8707          & \textbf{0.9113} & \textbf{0.0005} & \textbf{0.9469} \\
            \bottomrule
        \end{tabular}
    }
\end{table*}

\begin{table}[!t]
    \centering
    \caption{
        \textbf{Evaluation of molecular property prediction on QM9.}
        We report MAE on five properties:
        HOMO energy ($\varepsilon_{\mathrm{HOMO}}$),
        dipole moment ($\mu$),
        electronic spatial extent ($R^2$),
        polarizability ($\alpha$),
        and zero-point vibrational energy (ZPVE).
        \textbf{A.R.} denotes the average ranking across all five properties
        (lower is better; ties receive averaged ranks).
    }
    \label{tab:pp_main}
    \resizebox{\linewidth}{!}{
        \begin{tabular}{l | ccccc | c}
            \toprule
            \multirow{2}{*}{Method} & $\varepsilon_{\mathrm{HOMO}}$ & $\mu$      & $R^2$         & $\alpha$      & ZPVE          & \multirow{2}{*}{A.R.} \\
                                    & meV                           & mD         & m$\alpha_0^2$ & m$\alpha_0^3$ & meV           &                       \\
            \midrule
            SchNet                  & 41                            & 33         & 73            & 235           & 1.70          & 13.8                  \\
            Molformer               & 25                            & 28         & 350           & 41            & 2.05          & 11.5                  \\
            Cormorant               & 34                            & 38         & 961           & 85            & 2.03          & 16.8                  \\
            SEGNN                   & 24                            & 23         & 660           & 60            & 1.62          & 12.2                  \\
            EGNN                    & 29                            & 29         & 106           & 71            & 1.55          & 11.8                  \\
            PaiNN                   & 28                            & 12         & 66            & 45            & 1.28          & 7.9                   \\
            EQGAT                   & 20                            & 11         & 382           & 53            & 2.00          & 9.9                   \\
            NoisyNodes              & 20                            & 25         & 700           & 52            & 1.16          & 9.6                   \\
            GNS-TAT+NN              & 17                            & 21         & 650           & 47            & \textbf{1.08} & 7.8                   \\
            DimeNet++               & 25                            & 30         & 331           & 44            & 1.21          & 10.0                  \\
            SphereNet               & 23                            & 26         & 292           & 46            & 1.12          & 8.5                   \\
            MACE                    & 22                            & 15         & 210           & 38            & 1.23          & 6.9                   \\
            GA-GNN                  & 21                            & 11         & 63            & 45            & 1.18          & 5.1                   \\
            \midrule
            LieConv                 & 30                            & 32         & 800           & 84            & 2.28          & 16.2                  \\
            Equiformer v1           & 15                            & 11         & 251           & 46            & 1.26          & 6.9                   \\
            Equiformer v2           & \textbf{12}                   & \textbf{9} & 182           & 39            & 1.21          & 4.0                   \\
            \midrule
            FuncMol                 & 46                            & 15         & 174           & 36            & 3.64          & 10.3                  \\
            \method (Ours)          & 14                            & \textbf{9} & \textbf{62}   & \textbf{21}   & 1.16          & \textbf{1.8}          \\
            \bottomrule
        \end{tabular}
    }
\end{table}

\begin{table}[t]
    \centering
    \caption{
        \textbf{Ablation study on key components of \method.}
        Removing any component consistently degrades performance.
    }
    \label{tab:ablation}
    \resizebox{\linewidth}{!}{
        \begin{tabular}{l|cc|cc}
            \toprule
                                        & \multicolumn{2}{c|}{Dynamics} & \multicolumn{2}{c}{Property}                                         \\ \cmidrule(lr){2-3}\cmidrule(lr){4-5}
            Variant                     & IoU$\uparrow$                 & NC$\uparrow$                 & $\mu$$\downarrow$ & $R^2$$\downarrow$ \\ \midrule
            Full model (\method)        & \textbf{0.9113}               & \textbf{0.9469}              & \textbf{9}        & \textbf{63}       \\ \midrule
            \multicolumn{5}{l}{\textit{(a) C-INR Analysis}}                                                                                    \\ \midrule
            \quad w/o Canon.            & 0.8560                        & 0.8980                       & 23                & 146               \\
            \quad Rand. Proj.           & 0.8066                        & 0.8532                       & 35                & 285               \\ \midrule
            \multicolumn{5}{l}{\textit{(b) SWT Analysis}}                                                                                      \\ \midrule
            \quad Flatten Weights       & 0.8925                        & 0.9265                       & 12                & 77                \\
            \quad w/o PE                & 0.9025                        & 0.9323                       & 14                & 72                \\ \midrule
            \multicolumn{5}{l}{\textit{(c) FSHN Analysis}}                                                                                     \\ \midrule
            \quad Transformer $\to$ MLP & 0.8795                        & 0.9125                       & 14                & 91                \\
            \quad w/o Auto-Regressive   & 0.9010                        & 0.9350                       & 11                & 74                \\
            \bottomrule
        \end{tabular}
    }
\end{table}

\paragraph{Molecular Dynamics.}
Following \citet{sun2023dsr}, we evaluate \method on molecular dynamics simulations to assess its ability to model time-varying molecular structures in a continuous and resolution-independent manner.
Experiments are conducted on multiple protein trajectories from \citet{sun2023dsr}, covering diverse motion regimes from smooth conformational transitions to highly fluctuating equilibrium dynamics.
We compare against NDF \citep{sun2022topology}, DSR \citep{sun2023dsr}, and CanFields \citep{wang2025canfields}, using volumetric IoU, Chamfer Distance (CD), and Normal Consistency (NC) as evaluation metrics.
Table~\ref{tab:md_main} reports quantitative results on dynamic molecular surface reconstruction.
Across all systems, \method achieves the best average performance, with consistent improvements in both volumetric accuracy and surface normal alignment.
The gains are most pronounced on trajectories with substantial conformational motion (e.g., YiiP, NahA, I-FABP, and adk\_equi), indicating more accurate recovery of evolving molecular geometry.
Compared to DSR, which optimizes an implicit network per trajectory, \method benefits from SE(3)-invariant function-space generation, enabling it to capture shared dynamical structure and reduce optimization-induced distortions, particularly in complex motion regimes.

\paragraph{Molecular Property Prediction.}
We evaluate \method on QM9 \citep{ramakrishnan2014quantum}, following the standard 100K/18K/13K train/validation/test split of \citet{anderson2019cormorant}.
Five scalar targets are considered: $\varepsilon_{\mathrm{HOMO}}$, $\mu$, $R^2$, $\alpha$, and ZPVE, with mean absolute error (MAE) as the evaluation metric.
Baselines include: (1) GNN-based methods, including SchNet~\citep{schutt2017schnet}, Molformer~\citep{wu2023molformer}, Cormorant~\citep{anderson2019cormorant}, SEGNN~\citep{brandstettergeometric}, EGNN~\citep{satorras2021n}, PaiNN~\citep{schutt2021equivariant}, EQGAT~\citep{liao2022equiformer}, NoisyNodes~\citep{godwinsimple}, GNS-TAT~\citep{zaidi2022pre} DimeNet++~\citep{gasteiger2020fast}, SphereNet~\citep{liu2021spherical}, MACE~\citep{batatia2022mace}, GA-GNN~\citep{petersen2026design}; (2) CNN and transformer-based architectures, such as LieConv~\citep{batatia2023general} and Equiformer~\citep{liao2022equiformer,liao2023equiformerv2}; and (3) function-based approaches, represented by FuncMol~\cite{kirchmeyer2024score}.
As shown in Table~\ref{tab:pp_main}, \method achieves the lowest MAE on spatial distribution–related properties, including $\mu$ (9 mD), $R^2$ (62 m$\alpha_0^2$), and $\alpha$ (21 m$\alpha_0^3$), consistent with its continuous molecular-field representation.
It remains competitive on $\varepsilon_{\mathrm{HOMO}}$ (14 meV) and matches strong baselines on ZPVE (1.16 meV), indicating no degradation on properties dominated by local bonding patterns.
These results suggest that function-space parameterization effectively captures global spatial structure while maintaining stable geometric generalization.

\paragraph{Ablation Study.}
Table~\ref{tab:ablation} analyzes the contribution of each core component in \method.
Removing any component consistently degrades performance on both molecular dynamics and property prediction, indicating that the framework relies on their joint design.
\textit{(1) C-INR.}
Removing canonical coordination or replacing canonical coordination with random projections leads to substantial performance drops, confirming the importance of SE(3)-consistent function evaluation.
\textit{(2) SWT.}
Flattening weights or removing positional encodings degrades accuracy, showing that preserving the structured organization of function parameters is critical.
\textit{(3) FSHN.}
Replacing the Transformer with an MLP or removing autoregressive generation consistently hurts performance, highlighting the importance of modeling dependencies among function components.

\paragraph{Discussion on Molecular Generation.}

While our primary focus is on representation learning and function-space modeling,
the proposed framework also admits a natural extension to molecular generation.
We provide a preliminary generative formulation and qualitative results in Appendix~\ref{sec:generation}.

\begin{figure}[!t]
    \centering

    \subfigure{\includegraphics[width=0.49\linewidth]{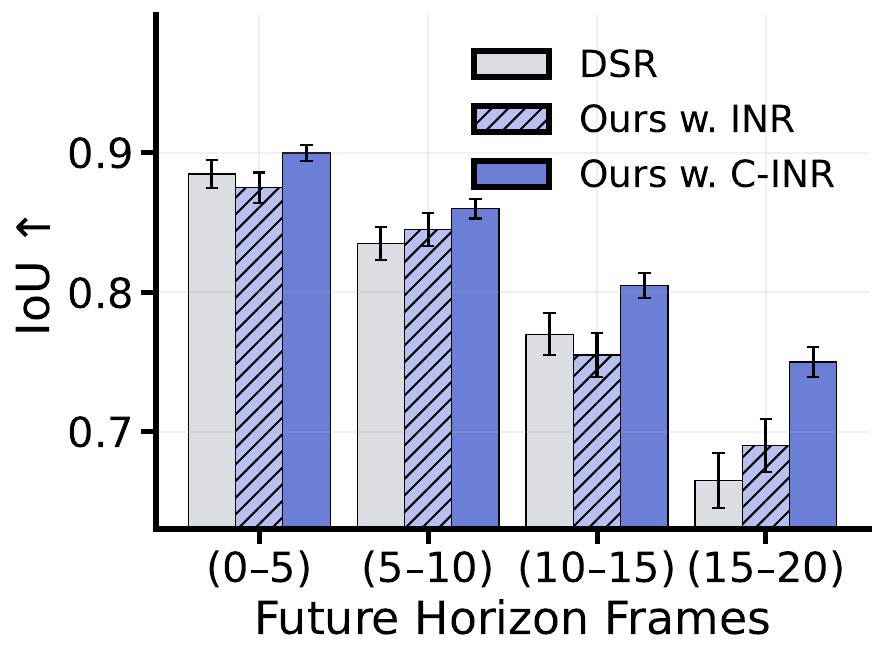}}
    \subfigure{\includegraphics[width=0.49\linewidth]{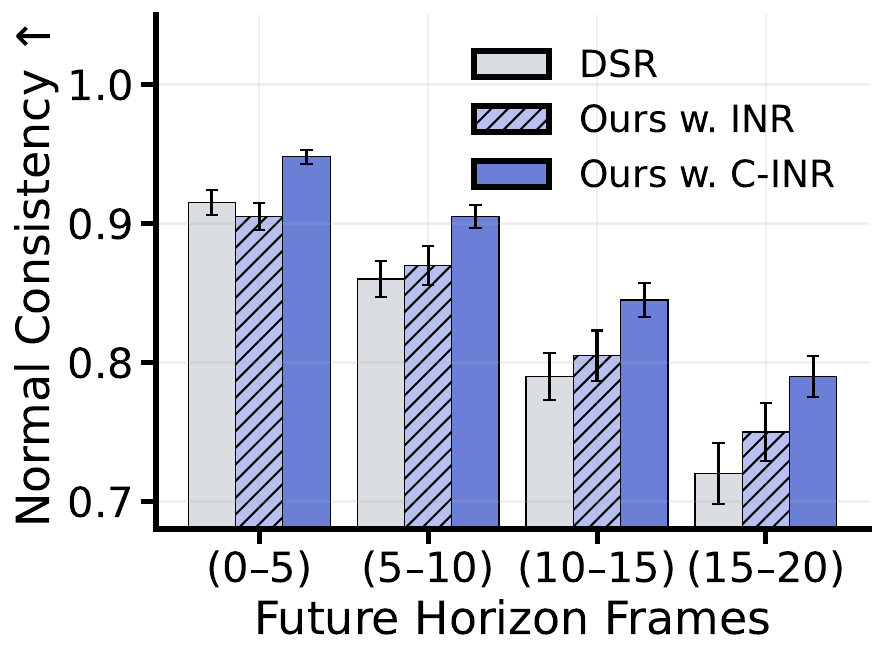}}
    \caption{
        \textbf{Long-horizon prediction on molecular dynamics.}
        Quantitative comparison of future surface prediction over increasing time horizons. \method (ours) yields more accurate and stable long-term predictions, particularly for extended horizons.
    }
    \label{fig:md_long_horizon}
\end{figure}

\begin{figure}[!t]
    \centering

    \subfigure{\includegraphics[width=0.49\linewidth]{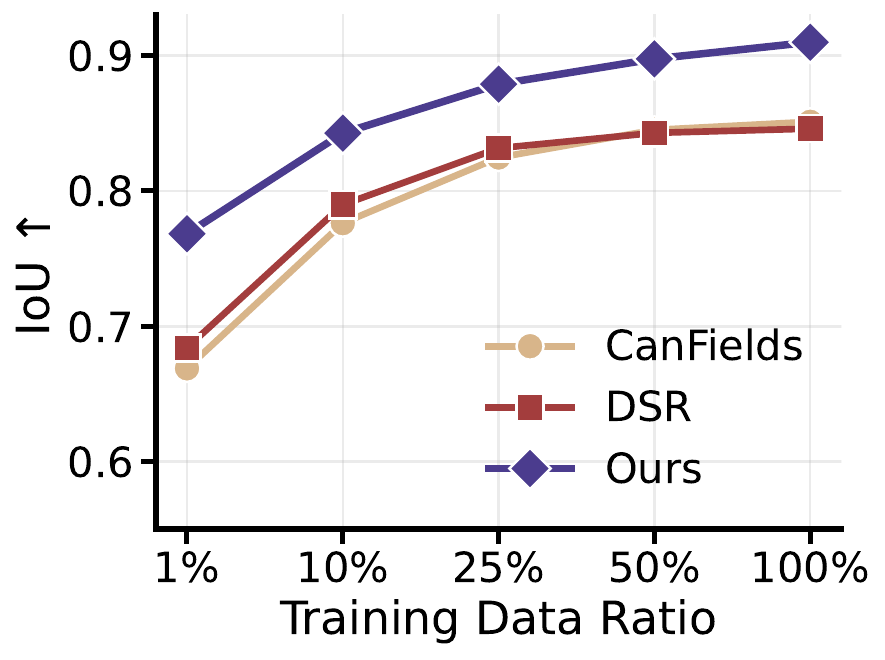}}
    \subfigure{\includegraphics[width=0.49\linewidth]{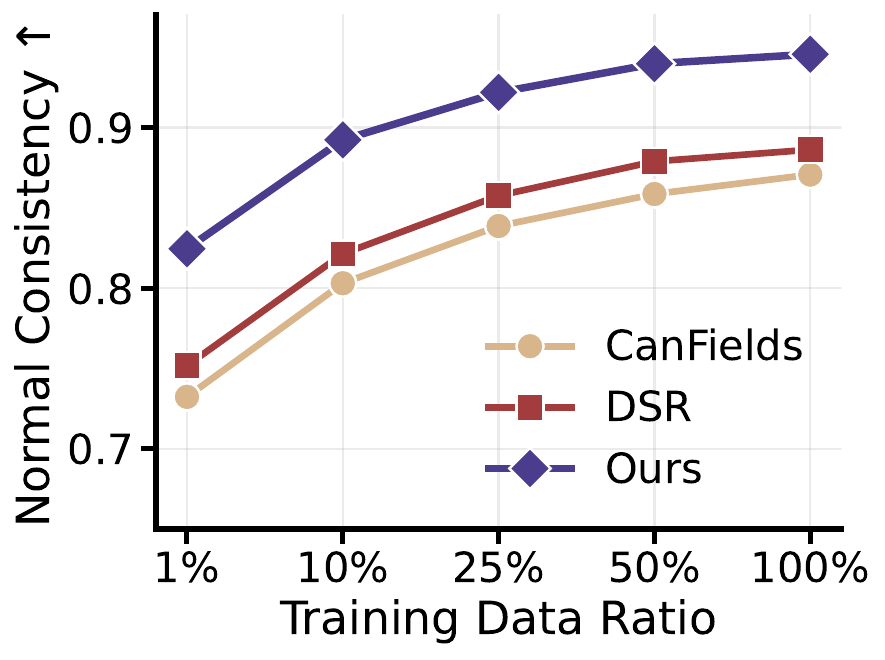}}
    \caption{
        \textbf{Data efficiency on molecular dynamics.}
        Performance of different methods under varying training data ratios. \method (ours) consistently achieves higher IoU and NC across all data regimes, demonstrating improved data efficiency.
    }
    \label{fig:md_data_efficiency}
\end{figure}

\section{Why Function-Space Representations of \method Generalize? }

We analyze why representing molecules as continuous functions leads to improved generalization across temporal horizons, data regimes, and downstream tasks.
Our central claim is that \method generalizes by learning a shared family of molecular functions, rather than task-specific discrete embeddings.
Under this formulation, molecular forms correspond to different query operators over the same underlying object, and generalization amounts to extrapolating or re-querying this object instead of relearning new representations.
The following analyses provide empirical evidence.

\paragraph{Long-Horizon Prediction as Function Extrapolation.}
To assess temporal generalization, we evaluate long-horizon prediction of molecular surfaces beyond the observed training window.
Figure~\ref{fig:md_long_horizon} reports IoU and Normal Consistency over increasing future horizons.
While DSR fits an implicit network independently for each trajectory, \method generates molecular fields from a learned distribution over functions.
This enables the model to extrapolate the underlying molecular field to unseen time steps, rather than relying on short-term geometric correlations.
As the prediction horizon increases, the performance gap widens, with canonical implicit representations maintaining substantially higher geometric fidelity for distant future frames.
These results indicate that \method captures shared dynamical structure at the level of molecular functions, instead of memorizing frame-wise surface variations, which is critical for stable long-term molecular dynamics modeling.

\paragraph{Data Efficiency from Sampling-Independent Representations.}
Figure~\ref{fig:md_data_efficiency} evaluates performance under varying fractions of training data.
\method consistently outperforms all baselines across data regimes, with especially large advantages under sparse supervision.
In discrete representations, reducing data directly alters the representation by changing the observed molecular structure.
In contrast, in function-space learning, sparse observations merely provide fewer queries of the same underlying object.
As a result, \method maintains stable performance even when the sampling density is low, achieving comparable accuracy with only 25\% of the data to DSR trained on the full dataset.
This behavior suggests that amortized generation of molecular functions enables the model to reuse knowledge across molecules and time steps, yielding a structured prior over molecular geometry that remains robust in low-data regimes.

\begin{figure}[!t]
    \centering
    \subfigure{\includegraphics[width=0.49\linewidth]{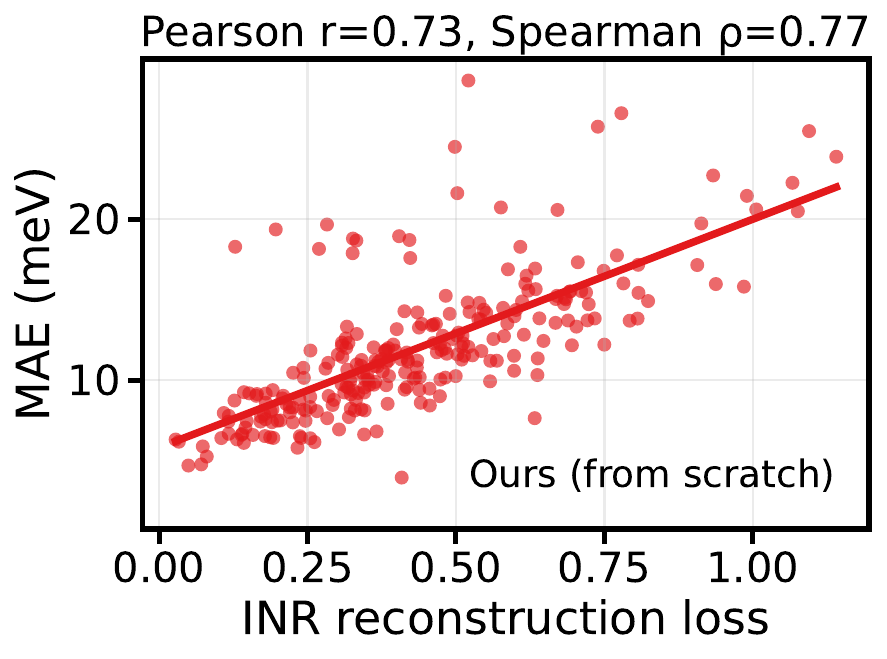}}
    \subfigure{\includegraphics[width=0.49\linewidth]{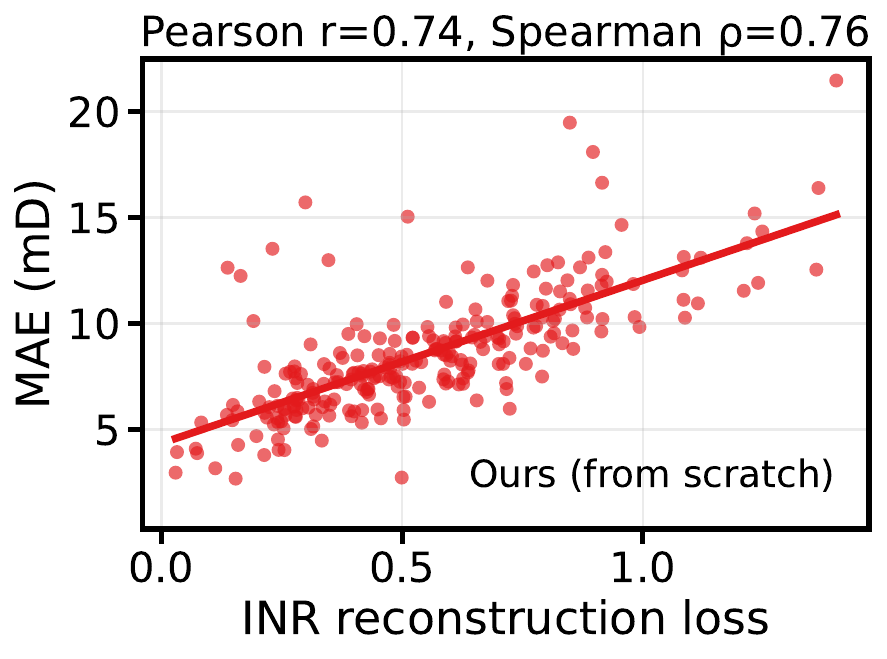}}

    \subfigure{\includegraphics[width=0.49\linewidth]{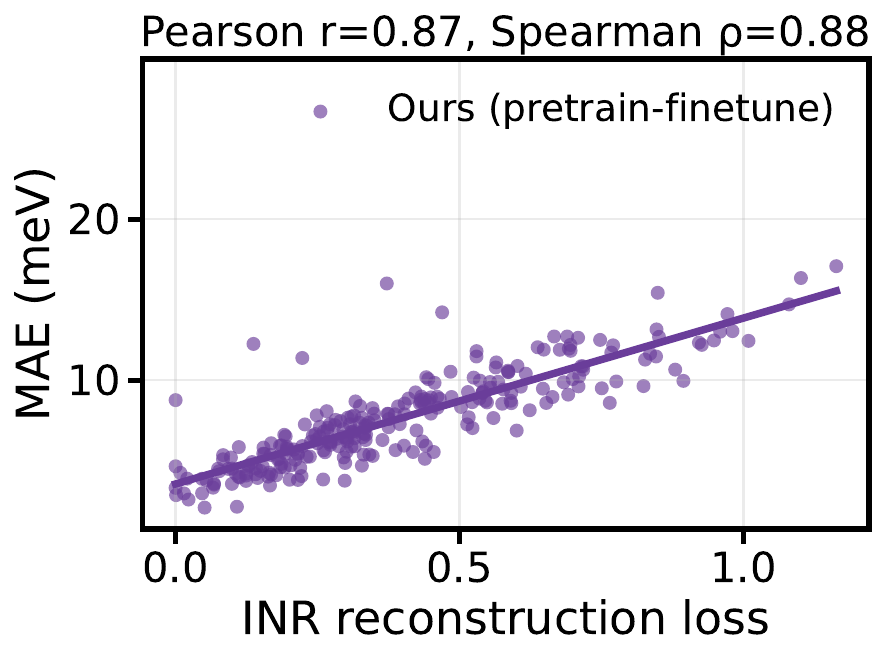}}
    \subfigure{\includegraphics[width=0.49\linewidth]{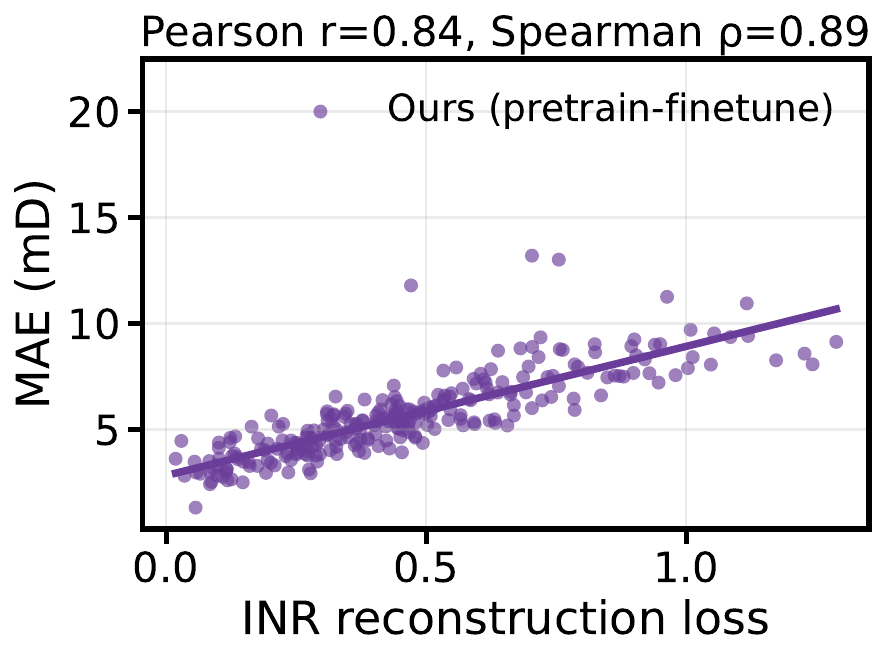}}

    \caption{
        \textbf{Correlation between INR reconstruction loss and property prediction error (MAE) of \method.}
        Strong Pearson and Spearman correlations indicate that more accurate function-space reconstructions lead to improved property prediction. Pretraining the INR on molecular generation further strengthens this relationship and consistently reduces prediction error.
    }
    \vspace{-15pt}
    \label{fig:pp_dist}
\end{figure}

\paragraph{Function Fidelity Mediates Downstream Performance.}
Figure~\ref{fig:pp_dist} analyzes the relationship between INR reconstruction loss on density prediction
and property prediction error.
Across all targets, we observe strong positive Pearson and Spearman correlations: molecules whose fields are reconstructed more accurately consistently yield lower prediction errors.
This trend holds both when training from scratch and under a pretrain--finetune regime, with pretraining further strengthening the correlation and concentrating predictions at lower MAE values.
These results indicate that property prediction in \method is mediated by the quality of the learned molecular function, rather than by task-specific shortcut features.
In other words, downstream performance improves because the shared function representation becomes more accurate, supporting the view that molecular learning proceeds through the function space itself.

\definecolor{dropbad}{RGB}{214,39,40}
\definecolor{dropgood}{RGB}{44,160,44}

\begin{table}[t]
    \centering
    \caption{
        \textbf{Robustness under input corruption}, measured by $\alpha$ (lower is better).
        Values in parentheses denote absolute degradation from the clean setting.
        Orig.~3D and MolField-3D operate on discrete graphs, while MolField-INR predicts directly in function space.
        Methods marked with * use Equiformer \citep{liao2022equiformer} as the backbone.
    }
    \label{tab:corruption_diagnosis}
    \resizebox{\linewidth}{!}{
        \begin{tabular}{l|ccc}
            \toprule
            Corruption & Orig.\ 3D*                           & MolField-3D*                        & MolField-INR                        \\
            \midrule
            0\%        & 46                                   & 25                                  & 21                                  \\ \midrule
            25\%       & 64 {\color{dropbad}\textbf{(↓18)}}   & 33 {\color{dropgood}\textbf{(↓8)}}  & 26 {\color{dropgood}\textbf{(↓5)}}  \\ \midrule
            50\%       & 91 {\color{dropbad}\textbf{(↓45)}}   & 54 {\color{dropgood}\textbf{(↓29)}} & 45 {\color{dropgood}\textbf{(↓24)}} \\ \midrule
            75\%       & 158 {\color{dropbad}\textbf{(↓112)}} & 68 {\color{dropgood}\textbf{(↓43)}} & 59 {\color{dropgood}\textbf{(↓38)}} \\
            \bottomrule
        \end{tabular}
    }

\end{table}


\paragraph{Robustness to Discretization.}
Table~\ref{tab:corruption_diagnosis} compares robustness under increasing input corruption for three settings: discrete 3D graphs, 3D graphs discretized from INR, and direct function-level prediction from INR.
Note that we use Equiformer \citep{liao2022equiformer} as the backbone for learning 3D structures.
As corruption increases, the performance over original 3D structures degrades rapidly, indicating high sensitivity to perturbations in discrete structural representations.
Discretizing from INR already reduces this degradation, suggesting that the continuous field acts as a stabilizing intermediate representation.
Most notably, direct prediction in function space (\method-INR) exhibits the smallest performance drop under all corruption levels.
Even at 75\% corruption, its absolute error increase remains substantially lower than both discrete and discretized counterparts.
This behavior supports our formulation that discrete molecular representations act as particular sampling or readout schemes of an underlying continuous object.
By operating directly on the function itself, \method avoids brittleness introduced by explicit discretization and achieves substantially greater robustness.
\section{Related Work}

\paragraph{Molecular Representations: Discrete and Continuous.}

Molecular representation fundamentally shapes what structural and physical information can be learned and how models generalize.
Most existing approaches rely on \emph{discrete surrogates}, including 1D sequence-based encodings such as SMILES or SELFIES \citep{weininger1988smiles,krenn2020self}, 2D graph-based message-passing networks \citep{gilmer2017neural,schutt2017schnet,gasteiger2020fast,schutt2021equivariant}, and 3D coordination-based equivariant networks \citep{satorras2021n,liao2022equiformer,batatia2022mace}.
Other works adopt continuous but still discretized representations, such as voxel grids or Gaussian density expansions \citep{liu2020neural,batatia2023general,wang2025canfields,li2024full,gainza2020deciphering}, which remain resolution-dependent and tied to specific sampling schemes.
In contrast, we treat the molecule itself as a continuous object and model it directly as a function over 3D space.
Rather than learning embeddings of particular discretizations (e.g., graphs or point clouds), we learn in function space and view discrete representations as query schemes over an underlying molecular field.
This shifts the representation object from task-specific structures to a shared continuous function, enabling transfer across tasks and discretizations.

\paragraph{Implicit Neural Representations for Molecules.}

Implicit neural representations model continuous signals as coordinate-based neural functions $f_\theta:\mathbf{x}\mapsto\mathbf{y}$ \citep{mildenhall2021nerf,martin2021nerf,essakine2024we}.
Several recent works apply INRs to molecular data, primarily for molecular surface reconstruction and dynamics modeling \citep{sverrisson2021fast,wu2024surfacevqmae,sun2023dsr,wu2025generalized}.
Another line of work represents molecules as continuous density fields for molecular generation and property prediction \citep{kirchmeyer2024score,lee2024pretraining}.
However, in these approaches, implicit functions are typically introduced as task-specific models or architectural choices, and are not treated as reusable molecular representations.
In contrast, we treat the implicit function itself as the molecular representation.
We learn a distribution over molecular functions via a hyper-network and reuse these functions across molecular dynamics, property prediction, and generation.
This elevates INRs from task-specific modeling techniques to a unified and transferable representation in function space.
\section{Conclusion}

We presented MolField, which formulates molecular learning as learning in function space by representing molecules as continuous fields rather than discrete structures. Under this view, graphs, point clouds, and sequences are different query schemes of the same underlying object.
This shift improves generalization, enabling stable long-horizon prediction, higher data efficiency, and robustness to discretization, with downstream performance governed by function fidelity.
We believe this perspective offers a principled foundation for molecular representation learning and a natural bridge between molecular modeling and neural field methods. 
\section*{Impact Statement}

This work introduces MolField, a function-space formulation of molecular representation learning that models molecules as continuous fields rather than discrete structures.
By shifting the representation object from task-specific embeddings to shared molecular functions, our framework aims to improve generalization, robustness, and transferability across molecular learning tasks.
As molecular machine learning increasingly influences scientific discovery and engineering practice, it is important to consider the broader implications of such representational advances.

\paragraph{Potential Benefits.}
MolField may benefit multiple scientific and engineering domains.
In computational chemistry \citep{jensen2017introduction} and drug discovery \citep{berdigaliyev2020overview}, more stable and transferable molecular representations could reduce data requirements, improve long-horizon simulation accuracy, and enable more reliable property prediction and molecular design.
In materials science \citep{wei2019machine}, function-space representations may facilitate modeling complex spatially continuous phenomena such as charge distributions or reaction fields.
More broadly, our formulation contributes to scientific machine learning by providing a principled interface between machine learning models and continuous physical systems, potentially enabling better integration of data-driven and physics-based approaches.

\paragraph{Risks and Misuse Considerations.}
As with many advances in molecular modeling, improved representation learning may lower barriers to molecular design and optimization.
In adversarial settings, such capabilities could be misused to assist in designing harmful chemical compounds or to accelerate unsafe experimentation.
While MolField itself does not introduce new generative objectives beyond existing molecular modeling frameworks, its improved generalization and robustness could amplify the effectiveness of downstream applications.
We therefore emphasize that responsible deployment should follow existing best practices in chemical safety, regulatory compliance, and dual-use risk assessment, and that access to high-capacity molecular design systems should be carefully managed.

\paragraph{Limitations and Scope.}
MolField focuses on improving molecular representations and does not directly address questions of chemical validity, synthetic feasibility, or biological safety of generated molecules.
Its benefits depend on the quality and diversity of available training data and do not substitute for experimental validation or domain expertise.
Furthermore, the current framework models static or moderately dynamic molecular fields and does not explicitly capture reactive processes or electronic dynamics at quantum-mechanical resolution \citep{merz2014using}.
These limitations constrain the scope of its applicability and should be considered when interpreting its results.

\paragraph{Future Directions.}
Future work may explore integrating function-space representations with physically grounded simulation frameworks, such as neural operators or differentiable solvers, to better capture complex molecular processes.
Extending the approach to larger biomolecular systems, reactive chemistry, or multimodal molecular data could further broaden its applicability.
From a societal perspective, developing technical safeguards, dataset governance strategies, and evaluation protocols for dual-use molecular modeling systems will be important to ensure that advances in representation learning are deployed responsibly and for beneficial scientific purposes.


\bibliography{citation}
\bibliographystyle{icml2026}


\newpage
\appendix
\onecolumn
\section{Additional Related Work}

\subsection{Implicit Neural Representation Generation}
Weight space generation studies how to directly model and generate the parameters of neural networks, rather than their outputs or latent embeddings. Early work introduces hypernetworks as conditional generators that map an input code or contextual signal to the weights of a target network~\citep{ha2017hypernetworks}. From this perspective, neural networks are treated as points in a structured weight space, where shared statistical regularities across tasks, instances, or signals can be captured and exploited. This formulation enables generalization at the level of functions themselves, and has been explored in contexts such as meta-learning, conditional computation, and neural function priors~\citep{finn2017model,wilson2020bayesian}.
Within this broader framework, INRs provide a natural and challenging testbed for weight space generation. An INR encodes a continuous signal as a neural function whose parameters fully determine the represented object. Classical INR construction relies on per-instance optimization via gradient descent, which is computationally expensive and often unstable under sparse or noisy observations~\citep{park2019deepsdf,sitzmann2020implicit}. Weight space generation offers an alternative by amortizing this optimization process: instead of iteratively fitting network parameters, a learned generator predicts the weights of an INR directly from observations, effectively instantiating a neural function in a single forward pass. This approach enables fast inference and highlights a principled connection between INRs, hypernetworks, and learned priors over neural functions~\citep{chen2022transformers,sitzmann2020metasdf}

\subsection{Molecule Generation}
Early molecular generative models primarily focused on 2D representations, treating molecules as SMILES strings or discrete graphs and modeling their distributions using VAEs~\citep{gomez2018automatic,jin2018junction}, GANs~\citep{de2018molgan}, normalizing flows~\citep{zang2020moflow} or more recently discrete diffusion models on graphs~\citep{vignac2023digress,liu2024graph,wang2025graph}. These approaches have achieved notable success in topology generation and property-conditioned design, but they do not explicitly model three-dimensional molecular geometry.
More recent work has shifted toward 3D molecule generation, where atoms are represented with continuous coordinates and categorical types. Representative approaches include normalizing flows~\citep{rezende2015variational,kohler2020equivariant,garcia2021n}, as well as a broad family of diffusion-based methods, encompassing equivariant diffusion models~\citep{hoogeboom2022equivariant,huang2023mdm}, flow matching formulations~\citep{song2023equivariant}, and score-based generative models~\citep{o20233d,zhao2025controllable}. Beyond point-cloud representations, recent work has further extended score-based generation to INRs, modeling molecules as continuous spatial fields rather than discrete atomic sets~\citep{kirchmeyer2024score}.

\subsection{Molecular Surface Reconstruction}

Molecular surfaces provide a compact and physically meaningful interface between biomolecules and their environment, encoding geometric and chemical cues that govern key interactions such as protein--ligand binding and allosteric regulation. Accurately modeling these surfaces as they evolve over time is therefore essential for understanding functional conformational changes and for enabling downstream tasks that rely on differentiable geometry.
Recent advances in dynamic molecular surface modeling leverage implicit neural representations to describe protein surfaces as continuous, time-dependent fields. Dynamical Surface Representation (DSR) models a protein surface as the zero-level set of a signed distance function (SDF) parameterized by an implicit neural network, with time explicitly incorporated as an input~\citep{sun2023dsr}. This formulation yields a continuous representation over both space and time, enabling smooth interpolation and extrapolation of surface dynamics. Training is performed directly from surface point clouds, without requiring precomputed SDF supervision, by enforcing manifold consistency and Eikonal regularization, while latent codes are used to capture protein-specific variability. To improve generalization across diverse protein families, subsequent work extends this framework with a mixture-of-experts architecture (MoE-DSR), where multiple implicit experts specialize in different structural distributions~\citep{wu2025generalized}. A learned routing mechanism extracts geometric information from the initial conformation to dynamically select experts, enhancing robustness to unseen proteins.

\subsection{Molecular Property Prediction}
Molecular property prediction aims to infer physicochemical and quantum-mechanical properties from molecular structure and has been extensively studied under a variety of representation and modeling paradigms. Graph neural networks constitute the dominant framework in this area, representing molecules as atom--bond graphs and learning molecular representations through iterative message passing~\citep{gilmer2017neural, schutt2017schnet,wang2024gft,wang2025towards,wang2025beyond,wang2025generative}.
To more faithfully capture three-dimensional molecular geometry, subsequent approaches explicitly incorporate geometric information such as interatomic distances, bond angles, and higher-order features, giving rise to directional message passing architectures exemplified by DimeNet and its variants~\citep{gasteiger2020directional, gasteiger2020fast}. More recent work enforces geometric symmetries via equivariant neural networks, including EGNN and SE(3)-equivariant architectures, which guarantee consistent transformations under rotations and translations~\citep{satorras2021n, fuchs2020se}. In parallel, transformer-based models have also been explored for molecular property prediction, leveraging global attention mechanisms to capture long-range interactions beyond local neighborhoods~\citep{sultan2024transformers, maziarka2020molecule}.

\section{Proofs} \label{sec:proof_symmetry}

\subsection{Proof of Proposition \ref{prop:frame_equi}} \label{sec:proof_of_frame_equivariance}
\begin{proof}
    The proof proceeds in three stages: the equivariance of learned vector features, the equivariance of aggregated global axes, and the preservation of equivariance under the Gram-Schmidt orthogonalization and cross-product completion.
    \paragraph{Equivariance of vector features.} The initial features $\mathbf{v}_i^{(0)}$ are constructed from relative coordinates $\tilde{\mathbf{x}}_j - \tilde{\mathbf{x}}_i$. Under a rotation $R$, the centered coordinates transform as $R\tilde{\mathbf{x}}$. Thus, the relative difference becomes $R(\tilde{\mathbf{x}}_j - \tilde{\mathbf{x}}_i)$. Since the initialization function $\psi$ and subsequent message-passing weights $\mathbf{W}_v, \eta$ act only on the invariant channel dimension (scalars), they commute with the spatial rotation matrix $R$. By induction, if $\mathbf{v}_i^{(\ell-1)}(RX) = R\mathbf{v}_i^{(\ell-1)}(X)$, then due to the linearity of the update rule, the refined features satisfy $\mathbf{v}_i^{(\ell)}(RX) = R \mathbf{v}_i^{(\ell)}(X)$ for all $\ell$.

    \paragraph{Equivariance of aggregated axes.} The global axes $\mathbf{u}_1, \mathbf{u}_2$ are obtained via weighted linear combinations of $\mathbf{v}_i$. Since the attention weights are derived from invariant scalars, they remain constant under rotation. Because matrix multiplication is a linear operator, the aggregation preserves equivariance:
    \begin{equation}
        \mathbf{u}_k(RX) = \sum_{i} \alpha_{i}^{(k)} R \mathbf{v}_i(X) = R \sum_{i} \alpha_{i}^{(k)} \mathbf{v}_i(X) = R \mathbf{u}_k(X), \quad \text{for } k \in \{1, 2\}.
    \end{equation}

    \paragraph{Equivariance of the orthonormal basis.}
    We verify the construction of the basis vectors $\mathbf{q}_1, \mathbf{q}_2, \mathbf{q}_3$ step-by-step:

    \begin{itemize}
        \item \textbf{First Axis ($\mathbf{q}_1$):} Normalization is scaling by a scalar length, which is rotation-invariant ($\|R\mathbf{u}\| = \|\mathbf{u}\|$). Thus:
              \begin{equation}
                  \mathbf{q}_1(RX) = \frac{R\mathbf{u}_1}{\|R\mathbf{u}_1\|} = R \frac{\mathbf{u}_1}{\|\mathbf{u}_1\|} = R\mathbf{q}_1(X).
              \end{equation}

        \item \textbf{Second Axis ($\mathbf{q}_2$ via Gram-Schmidt):} The projection step removes the component of $\mathbf{u}_2$ parallel to $\mathbf{q}_1$. Let $\tilde{\mathbf{u}}_2 = \mathbf{u}_2 - (\mathbf{u}_2^\top \mathbf{q}_1)\mathbf{q}_1$. Under rotation $R$:
              \begin{align}
                  \tilde{\mathbf{u}}_2(RX) & = R\mathbf{u}_2 - ((R\mathbf{u}_2)^\top (R\mathbf{q}_1)) (R\mathbf{q}_1)                       \\
                                           & = R\mathbf{u}_2 - (\mathbf{u}_2^\top R^\top R \mathbf{q}_1) (R\mathbf{q}_1)                    \\
                                           & = R\mathbf{u}_2 - (\mathbf{u}_2^\top \mathbf{I} \mathbf{q}_1) R\mathbf{q}_1                    \\
                                           & = R (\mathbf{u}_2 - (\mathbf{u}_2^\top \mathbf{q}_1)\mathbf{q}_1) = R \tilde{\mathbf{u}}_2(X).
              \end{align}
              Since $\tilde{\mathbf{u}}_2$ rotates equivariantly, its normalized version $\mathbf{q}_2$ also satisfies $\mathbf{q}_2(RX) = R\mathbf{q}_2(X)$.

        \item \textbf{Third Axis ($\mathbf{q}_3$ via Cross Product):} A critical property of the cross product for rotation matrices $R \in SO(3)$ (where $\det(R)=1$) is $R(\mathbf{a} \times \mathbf{b}) = (R\mathbf{a}) \times (R\mathbf{b})$. Thus:
              \begin{align}
                  \mathbf{q}_3(RX) & = \mathbf{q}_1(RX) \times \mathbf{q}_2(RX)                \\
                                   & = (R\mathbf{q}_1) \times (R\mathbf{q}_2)                  \\
                                   & = R(\mathbf{q}_1 \times \mathbf{q}_2) = R\mathbf{q}_3(X).
              \end{align}
              Note that this step strictly relies on the right-handed nature of the frame; for reflections, this equality would flip the sign.

    \end{itemize}

    Stacking the columns, we obtain:
    \begin{equation}
        Q(RX) = [R\mathbf{q}_1, R\mathbf{q}_2, R\mathbf{q}_3] = R [\mathbf{q}_1, \mathbf{q}_2, \mathbf{q}_3] = R Q(X).
    \end{equation}
\end{proof}

\paragraph{Discussion.}
The canonical frame construction assumes $\mathbf{u}_1 \neq \mathbf{0}$ and a non-degenerate orthogonalized $\tilde{\mathbf{u}}_2$.
Exact collapse of these vectors is non-generic under continuous parameterization.
Moreover, since the canonical axes are normalized and used end-to-end in the loss, near-collapse leads to ill-conditioned canonicalization: the normalization Jacobian scales as $O(1/\|\mathbf{u}\|)$, which typically amplifies gradients and steers optimization away from such regimes.
In practice, we further ensure numerical stability by applying $\epsilon$-regularized normalization, which guarantees well-defined canonical coordinates even in extreme cases.

\subsection{Proof of Lemma \ref{lemma:se3}} \label{sec:proof_of_invariant_coordinate_mapping}
\begin{proof}
    Let $X' = g \cdot X$ denote the transformed molecular geometry and $\mathbf{x}' = R\mathbf{x} + \mathbf{t}$ denote the transformed query point.
    First, observe that the centroid shifts equivariantly: $\bar{\mathbf{x}}' = \frac{1}{N}\sum (R\mathbf{x}_i + \mathbf{t}) = R\bar{\mathbf{x}} + \mathbf{t}$.
    Second, from Proposition~\ref{prop:frame_equi}, the frame rotates only: $Q(X') = R Q(X)$.
    Substituting these into the query definition:
    \begin{align}
        c(\mathbf{x}'; X') & = Q(X')^\top (\mathbf{x}' - \bar{\mathbf{x}}')                                                 \\
                           & = (R Q(X))^\top \left( (R\mathbf{x} + \mathbf{t}) - (R\bar{\mathbf{x}} + \mathbf{t}) \right)   \\
                           & = Q(X)^\top R^\top \left( R(\mathbf{x} - \bar{\mathbf{x}}) + (\mathbf{t} - \mathbf{t}) \right) \\
                           & = Q(X)^\top R^\top R (\mathbf{x} - \bar{\mathbf{x}})                                           \\
                           & = Q(X)^\top \mathbf{I} (\mathbf{x} - \bar{\mathbf{x}})                                         \\
                           & = c(\mathbf{x}; X).
    \end{align}
    This demonstrates that the canonical coordinate mapping implicitly handles translation via relative differences and rotation via frame projection, ensuring the downstream INR receives consistent inputs regardless of the global pose.
\end{proof}

\paragraph{Discussion.}
This invariance is crucial: the INR weights $\theta$ only learn to recognize features in the canonical frame $\mathbf{x}'$. During inference (e.g., surface reconstruction), we query the field at $\mathbf{x}$ by passing its canonical counterpart $\mathbf{x}'$ into the network. This allows the network output to track the molecule's rotation without the weights themselves being equivariant.

\section{Pseudo-Code
 }
\label{sec:pseudo-code}
\paragraph{End-to-End Training Pipeline.}
Algorithm~\ref{alg:molfield_unified} summarizes the end-to-end training procedure of \method in function space.
Given a training sample $(\mathcal{I}, y)$, the model first constructs a task-dependent condition $z$ from the input.
Conditioned on $z$, a transformer-based hyper-network autoregressively generates a structured sequence of weight tokens, which are subsequently detokenized into the parameters $\theta$ of a C-INR.
This process instantiates a continuous molecular field $f_{\theta}$ in a single forward pass, without per-instance optimization.

To train the generated field, \method follows a generate-then-query paradigm.
A set of spatial query points is sampled in $\mathbb{R}^3$ according to the task and input, and each query point is mapped to a canonical coordinate system defined by the molecular geometry.
The implicit field is evaluated at these canonical coordinates, producing field values as the final prediction.
A task loss is then computed by comparing the prediction with the supervision signal, and gradients are backpropagated through the C-INR queries and the hyper-network, enabling end-to-end optimization of the hyper-network parameters.

Importantly, in this illustrative single-task setting, the pipeline is fully specified by the condition $z$. Across tasks, the same pipeline applies with task-specific conditioning defined in the main text.
This unified formulation allows \method to learn a transferable prior over molecular functions, which can be reused across molecular dynamics, property prediction, and generation.

\begin{algorithm}[!t]
    \caption{\method: End-to-End Training in Function Space}
    \label{alg:molfield_unified}
    \begin{algorithmic}[1]
        \REQUIRE Dataset $\mathcal{D}=\{(\mathcal{I},y)\}$ for a fixed task; condition $z$ associated with $\mathcal{I}$;
        hyper-network $G_{\phi}$; C-INR with $L$ layers; detokenizers $\textsc{Proj}_{(\ell,r)}(\cdot)$;
        equivariant encoder $\Phi(\cdot)$; query sampler $S(z,\mathcal{I})$; loss $\ell(\cdot,\cdot)$
        \ENSURE Trained hyper-network parameters $\phi$

        \STATE Initialize $\phi$
        \WHILE{not converged}
        \STATE Sample mini-batch $\mathcal{B}=\{(\mathcal{I},y)\}\sim\mathcal{D}$
        \FORALL{$(\mathcal{I},y)\in\mathcal{B}$}
        \STATE \textbf{(1) Generate C-INR parameters via hyper-network}
        \FOR{$\ell=1$ \textbf{to} $L$}
        \FORALL{$r\in\{\texttt{weight},\texttt{bias}\}$}
        \STATE $\tau_t \sim p_{\phi}(\tau_t \mid \tau_{<t}, z)$
        \STATE $\theta^{(\ell,r)} \leftarrow \textsc{Proj}_{(\ell,r)}(\tau_t)$
        \ENDFOR
        \ENDFOR
        \STATE Instantiate field network $f_{\theta}$

        \STATE \textbf{(2) Sample one spatial query point}
        \STATE $\mathbf{x} \leftarrow S(z,\mathcal{I})$ \hfill // sample one 3D point (illustration)

        \STATE \textbf{(3) Canonical querying (C-INR)}
        \STATE Extract molecular geometry $(X,A)$ from $\mathcal{I}$
        \STATE $\bar{\mathbf{x}} \leftarrow \frac{1}{N}\sum_{i=1}^{N}\mathbf{x}_i$
        \STATE $\{\mathbf{v}_i\}\leftarrow \Phi(X,A)$
        \STATE Build canonical frame $Q(X)\in SO(3)$ from $\{\mathbf{v}_i\}$
        \STATE $\mathbf{x}' \leftarrow Q(X)^{\top}(\mathbf{x}-\bar{\mathbf{x}})$
        \STATE $\hat{y} \leftarrow f_{\theta}(\mathbf{x}')$ \hfill // queried field value

        \STATE \textbf{(4) End-to-end optimization}
        \STATE $L \leftarrow \ell(\hat{y}, y)$
        \STATE Update $\phi$ by backpropagating $\nabla_{\phi} L$
        \ENDFOR
        \ENDWHILE
        \STATE \textbf{return} $\phi$
    \end{algorithmic}
\end{algorithm}

\section{Equivariant Encoder Details}
\label{sec:encoder}

Our encoder is designed to extract SE(3)-invariant scalar features $H = \{h_i\}_{i=1}^N \in \mathbb{R}^{N \times C}$ and SO(3)-equivariant vector features $V = \{\mathbf{v}_i\}_{i=1}^N \in \mathbb{R}^{N \times C \times 3}$ from the input molecular graph. The architecture follows a dual-stream message passing scheme, ensuring that geometric information propagates efficiently while strictly preserving equivariance.

\subsection{Initialization}The input consists of atomic numbers $Z$ and positions $X$. We perform the following initializations before the first layer:

\paragraph{Scalar initialization.} The scalar features $h_i^{(0)}$ are initialized by embedding the atomic numbers using a learnable lookup table:
\begin{equation}
    h_i^{(0)} = \text{Embedding}(Z_i) \in \mathbb{R}^{C}.
\end{equation}

\paragraph{Edge embedding.} To capture pairwise geometry, we compute the relative distance $d_{ij} = \|\tilde{\mathbf{x}}_j - \tilde{\mathbf{x}}_i\|$ and encode it using a set of $K$ Radial Basis Functions (RBF), followed by a linear projection:
\begin{equation}
    e_{ij} = \text{Linear}(\text{RBF}(d_{ij})) \in \mathbb{R}^{C}.
\end{equation}
This edge embedding $e_{ij}$ is invariant to SE(3) transformations and acts as a geometric bias in message passing.

\paragraph{Vector initialization.} We initialize the vector features $\mathbf{v}_i^{(0)}$ by aggregating local bond directions. This explicitly couples the vector stream with the molecular geometry:
\begin{equation}
    \mathbf{v}_i^{(0)} = \sum_{j \in \mathcal{N}(i)} \psi(h_i^{(0)}, h_j^{(0)}) \cdot \frac{\tilde{\mathbf{x}}_j - \tilde{\mathbf{x}}_i}{|\tilde{\mathbf{x}}_j - \tilde{\mathbf{x}}_i|},
\end{equation}
where $\psi: \mathbb{R}^{2C} \to \mathbb{R}^{C}$ is a shallow MLP that weights the bond directions based on the chemical species of the atom pair.

\subsection{Equivariant Message Passing Layer}Each layer $\ell$ updates both scalar and vector features. The update involves two stages: interaction-aware message passing and channel-wise mixing.

\paragraph{Vector update (equivariant stream).} We update the vector features by aggregating information from neighbors. Crucially, we preserve rotation equivariance by restricting operations to linear combinations of existing vectors. The update rule is:
\begin{equation}
    \mathbf{v}_i^{(\ell)} = \mathbf{W}_{v} \mathbf{v}_i^{(\ell-1)} + \sum_{j \in \mathcal{N}(i)} \eta_{ij}^{(\ell)} \cdot \mathbf{v}_{j}^{(\ell-1)},
\end{equation}
where $\mathbf{W}_{v} \in \mathbb{R}^{C \times C}$ is a learnable channel-mixing matrix. The coefficients $\eta_{ij}^{(\ell)}$ are scalar gating values derived from the invariant stream, allowing scalar information to modulate vector propagation:
\begin{equation}
    \eta_{ij}^{(\ell)} = \text{MLP}_{gate}\left( h_i^{(\ell-1)}, h_j^{(\ell-1)}, e_{ij} \right) \in \mathbb{R}^{C}.
\end{equation}
This mechanism ensures that the vector aggregation is sensitive to both chemical context and scalar distances.

\paragraph{Scalar update (invariant stream).} The scalar features are updated by aggregating messages from neighbors and, importantly, fusing information from the vector stream. This allows angular and directional information (captured by vectors) to influence the scalar chemical properties:
\begin{equation}
    m_{ij}^{(\ell)} = \phi_{msg}\left( h_i^{(\ell-1)}, h_j^{(\ell-1)}, e_{ij} \right),
\end{equation}
\begin{equation}
    h_i^{(\ell)} = h_i^{(\ell-1)} + \text{MLP}_{upd}\left( \sum_{j \in \mathcal{N}(i)}m_{ij}^{(\ell)}\oplus \|\mathbf{v}_i^{(\ell)}\| \right).
\end{equation}
Here, $\phi_{msg}$ is the message function, and $\|\mathbf{v}_i^{(\ell)}\| \in \mathbb{R}^{C}$ denotes the channel-wise norm of the updated vector features. By concatenating the vector norms, we allow the scalar features to sense the anisotropy of the local environment without breaking rotation invariance (since the norm of a vector is rotation-invariant).

\section{Experimental Details}

\subsection{Molecular Dynamics}

Molecular dynamics simulations provide a physically grounded and high-resolution description of protein conformational dynamics, and are therefore widely used as benchmarks for evaluating data-driven models of biomolecular motion. In this work, we follow \citet{sun2023dsr} to construct a comprehensive benchmark suite by combining multiple publicly available MD datasets, covering diverse protein sizes, structural classes, and dynamical regimes. Specifically, we include long-timescale trajectories of intrinsically disordered proteins, soluble enzymes undergoing equilibrium fluctuations and conformational transitions, as well as large membrane proteins with complex collective motions. To further assess generalization across related systems, we also incorporate multiple trajectories from the same protein families simulated under similar protocols. For computational efficiency and consistency, extremely long trajectories are temporally cropped to a fixed maximum length, while shorter trajectories are used in full. This combination of datasets enables a systematic evaluation of model performance in terms of surface reconstruction accuracy, temporal interpolation, and future-state prediction under realistic and heterogeneous protein dynamics.
We follow \citet{sun2023dsr} for the experimental setup.

\subsection{Molecular Property Prediction}

In addition to protein dynamics benchmarks, we evaluate our model on the QM9 dataset for molecular property prediction \citep{ramakrishnan2014quantum}. QM9 is a widely used quantum chemistry benchmark consisting of approximately 134k small organic molecules with up to nine heavy atoms, where molecular geometries and properties are computed using density functional theory. Following standard practice, we consider five representative target properties that capture both electronic and structural characteristics of molecules: the highest occupied molecular orbital energy ($\varepsilon_{\text{HOMO}}$), dipole moment ($\mu$), electronic spatial extent ($R^2$), isotropic polarizability ($\alpha$), and zero-point vibrational energy (ZPVE). These properties span different physical scales and units, providing a comprehensive assessment of a model’s ability to learn quantum-mechanical observables from molecular structures. Performance is reported using mean absolute error for each target, and an average ranking across all properties is additionally used to summarize overall predictive accuracy.
We follow the \citet{anderson2019cormorant} for the experimental setup.

\subsection{Molecular Generation}

We further evaluate our model on unconditional three-dimensional molecular generation tasks using the QM9 \citep{ramakrishnan2014quantum} and Drugs \citep{axelrod2022geom} datasets. QM9 contains small organic molecules with up to nine heavy atoms and well-optimized equilibrium geometries, serving as a standard benchmark for learning physically valid 3D molecular distributions. The Drugs dataset consists of substantially larger and more diverse drug-like compounds, providing a more challenging setting with increased structural complexity and conformational variability.
We adopt the standard unconditional generation protocol, where models are trained to learn the joint distribution of molecular topology and atomic coordinates from the training set without conditioning on molecular properties or target structures. During training, molecules are represented by their atomic types and 3D coordinates, and the model is optimized to maximize the likelihood (or equivalently minimize the training objective defined by the generative framework) over the observed molecular structures. At inference time, molecules are generated from random latent samples, and full molecular geometries are produced in a single forward sampling process.
We follow \citet{hoogeboom2022equivariant,kirchmeyer2024score} for the experimental setup.

\subsection{Implementation Details of \method}

\method employs a Transformer-based hyper-network with 6 encoder layers, 8 attention heads, hidden dimension 2048, and dropout rate 0.1 generates weights for the C-INR, which is implemented as an 8-layer MLP with hidden dimensions [512, 512, 512, 512, 512, 512, 512, 512] and skip connections at layer 4.
For the equivariant encoder, \method uses an equivariant Graph Neural Network with 6 message-passing layers, hidden dimension 512, and edge features encoded using radial basis functions with 32 centers.
For the molecular dynamics task, the time coordinates are encoded using 64-dimensional Fourier features projected to 512 dimensions.
We train for 2000 epochs using Adam optimizer with initial learning rates $10^{-4}$ for the network and $10^{-3}$ for latent codes, both decayed by factor 0.5 every 200 epochs, with batch size 1 and 16 data loading workers. All experiments are conducted on NVIDIA 4 $\times$ A40 GPUs using PyTorch 1.12.0 and PyTorch Geometric.

\section{Additional Experiments
 }

\subsection{Molecule Generation}
\label{sec:generation}





In this section, we present supplementary results on molecular generation to illustrate that the proposed function-space formulation admits a coherent generative interpretation.
We stress that molecular generation is \emph{not} a primary evaluation target of \method, nor is the framework optimized specifically for this task.
Instead, these experiments serve to demonstrate the extensibility of learning molecular representations directly in implicit functional space, beyond the discriminative and predictive settings studied in the main body.

\begin{table}[!t]
    \centering
    \caption{
        \textbf{Evaluation on QM9 molecular generation.}
        We report validity and uniqueness (Valid \& Uni), atom-level stability (AtomSta), and molecule-level stability (MolSta).
        Our method achieves competitive performance across all metrics, while maintaining strong molecular stability compared to existing approaches and the current state-of-the-art \citep{you2024latent}.
    }
    \label{tab:generation_qm9}
    \begin{tabular}{l|ccc}
        \toprule
        Method        & Valid \& Uni (\%) & AtomSta (\%)  & MolSta (\%)   \\
        \midrule
        ENF           & 39.4              & 85.0          & 4.9           \\
        G-SchNet      & 80.3              & 95.7          & 68.1          \\
        GDM           & --                & 97.0          & 63.2          \\
        EDM           & 90.7              & \textbf{98.7} & 82.6          \\ \midrule
        \textbf{Ours} & \textbf{91.3}     & 97.2          & \textbf{83.7} \\ \midrule
        Current SOTA  & 95.3              & 97.6          & 87.7          \\
        \bottomrule
    \end{tabular}
\end{table}

\paragraph{Experimental setting.}
We follow the experimental protocols of \citet{kirchmeyer2024score} and evaluate molecular generation on QM9~\citep{ramakrishnan2014quantum} and Drugs~\citep{axelrod2022geom}.
Generation in \method is performed by sampling latent conditions from a Gaussian prior and instantiating continuous molecular fields through the learned hyper-network.
Discrete molecular graphs are then recovered via standard post-processing procedures applied to the generated fields.
We report commonly used generation metrics, including validity and uniqueness (Valid \& Uni), atom-level stability (AtomSta), and molecule-level stability (MolSta).
Comparisons are made against representative graph-based and continuous-space generative models, including ENF~\citep{garcia2021n}, G-SchNet~\citep{gebauer2019symmetry}, GDM~\citep{hoogeboom2022equivariant}, EDM~\citep{hoogeboom2022equivariant}, and FuncMol~\citep{kirchmeyer2024score}, using their standard reported configurations.

\paragraph{Generation results.}
Table~\ref{tab:generation_qm9} reports quantitative results on molecular generation.
While \method is not designed as a specialized molecular generative model, it achieves competitive performance across multiple metrics, including validity, uniqueness, and molecular stability.
This indicates that the learned function-space prior captures non-trivial molecular structure and produces coherent samples rather than degenerate or noisy outputs.
Importantly, these results suggest that a representation primarily learned for molecular dynamics modeling and property prediction can nevertheless support meaningful generative sampling, without introducing task-specific generation modules.
Overall, the results provide evidence that the proposed functional parameterization is expressive enough to admit a viable generative interpretation.

\definecolor{drop-1}{RGB}{225,230,245}
\definecolor{drop-2}{RGB}{245,235,235}

\begin{table}[!t]
    \centering
    \caption{
        \textbf{Transfer stability on Drugs.}
        Atom-level and molecule-level performance under scratch and transfer settings, together with the performance drop (Drop $\downarrow$), where lower is better.
        \method shows substantially smaller degradation compared to prior methods, demonstrating improved robustness to domain shift.
    }
    \label{tab:stability_transfer_transposed}
    \begin{tabular}{l | ccc | ccc}
        \toprule
        \multirow{2}{*}{Method}
         & \multicolumn{3}{c|}{AtomSta}
         & \multicolumn{3}{c}{MolSta}                                               \\
        \cmidrule(lr){2-4} \cmidrule(lr){5-7}
         & Scratch                      & Transfer & Drop $\downarrow$
         & Scratch                      & Transfer & Drop $\downarrow$              \\
        \midrule
        EDM
         & 97.8                         & 91.5     & \cellcolor{drop-2}\textbf{6.3}
         & 40.3                         & 36.7     & \cellcolor{drop-2}\textbf{3.6} \\

        FuncMol
         & 95.3                         & 87.6     & \cellcolor{drop-2}\textbf{7.7}
         & 69.7                         & 60.1     & \cellcolor{drop-2}\textbf{9.6} \\
        \midrule

        \method
         & 90.1                         & 88.4     & \cellcolor{drop-1}\textbf{1.7}
         & 56.3                         & 54.7     & \cellcolor{drop-1}\textbf{1.6} \\
        \bottomrule
    \end{tabular}
\end{table}

\paragraph{Transferability under distribution shift.}
We further consider a transfer generation setting, where the model is pretrained on QM9 and finetuned on the Drugs dataset.
As shown in Table~\ref{tab:stability_transfer_transposed}, \method exhibits smaller degradation in atom-level and molecule-level stability under domain shift compared to prior approaches.
This behavior suggests that the learned function-space representation captures structural regularities that are transferable across molecular domains, rather than being tightly coupled to dataset-specific graph statistics.
Such transferability is consistent with the view of molecules as continuous functional objects and further supports the generality of the proposed formulation.



\begin{figure*}[htbp]
    \centering
    \subfigure{\includegraphics[width=0.19\linewidth]{fig/results/property_scatter_qm9_scratch_eps_HOMO_0.pdf}}
    \subfigure{\includegraphics[width=0.19\linewidth]{fig/results/property_scatter_qm9_scratch_mu_1.pdf}}
    \subfigure{\includegraphics[width=0.19\linewidth]{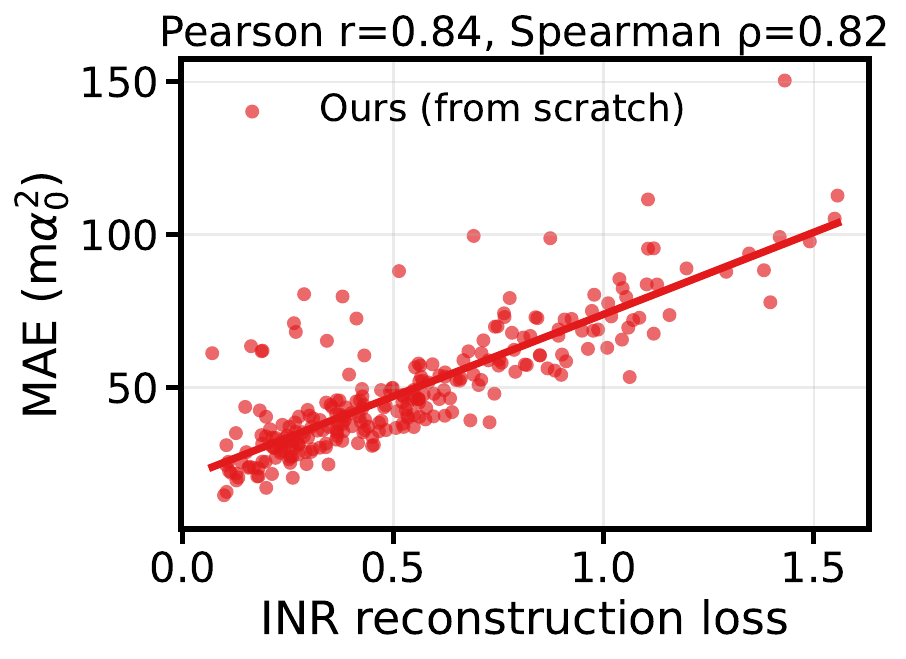}}
    \subfigure{\includegraphics[width=0.19\linewidth]{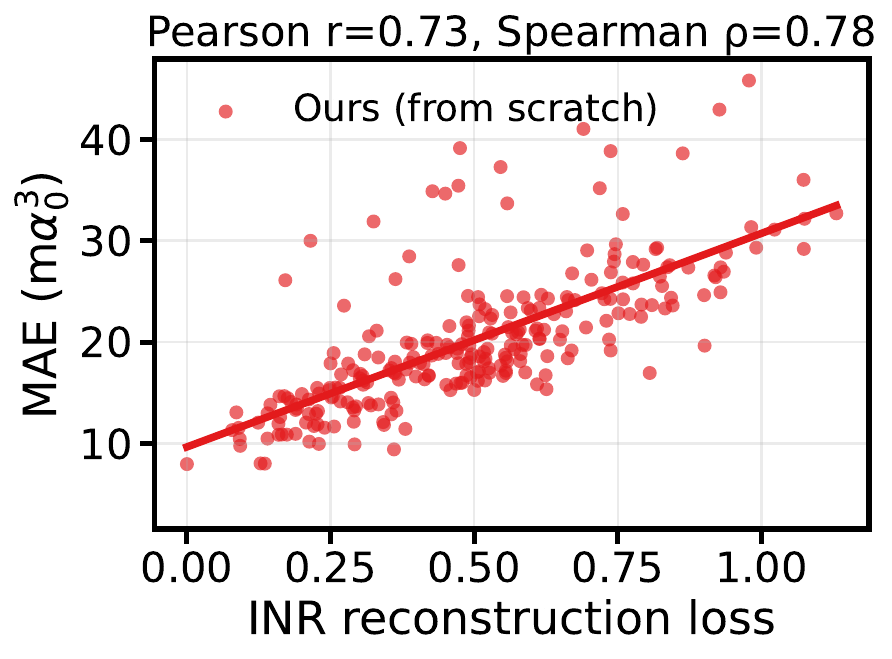}}
    \subfigure{\includegraphics[width=0.19\linewidth]{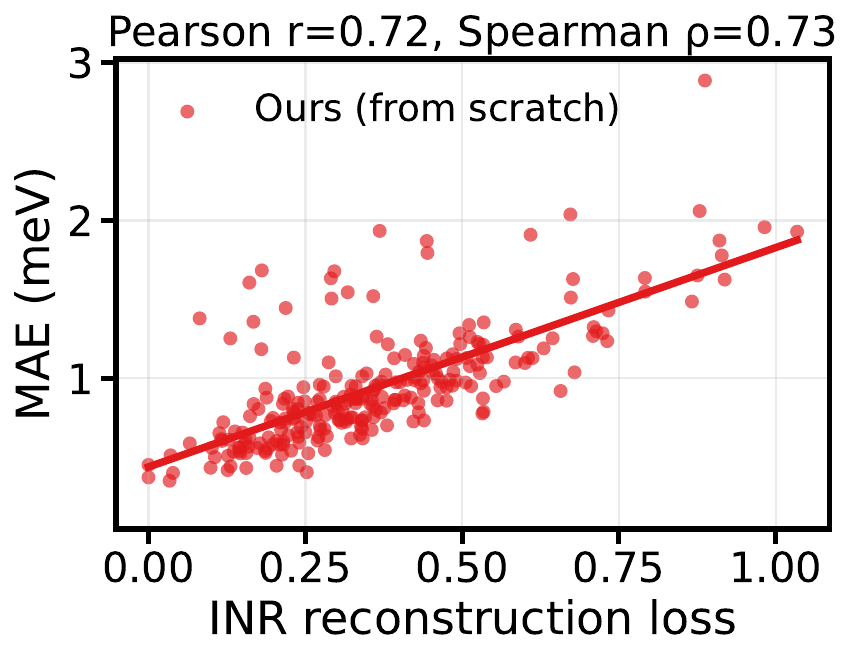}}

    \subfigure{\includegraphics[width=0.19\linewidth]{fig/results/property_scatter_qm9_pretrain_eps_HOMO_0.pdf}}
    \subfigure{\includegraphics[width=0.19\linewidth]{fig/results/property_scatter_qm9_pretrain_mu_1.pdf}}
    \subfigure{\includegraphics[width=0.19\linewidth]{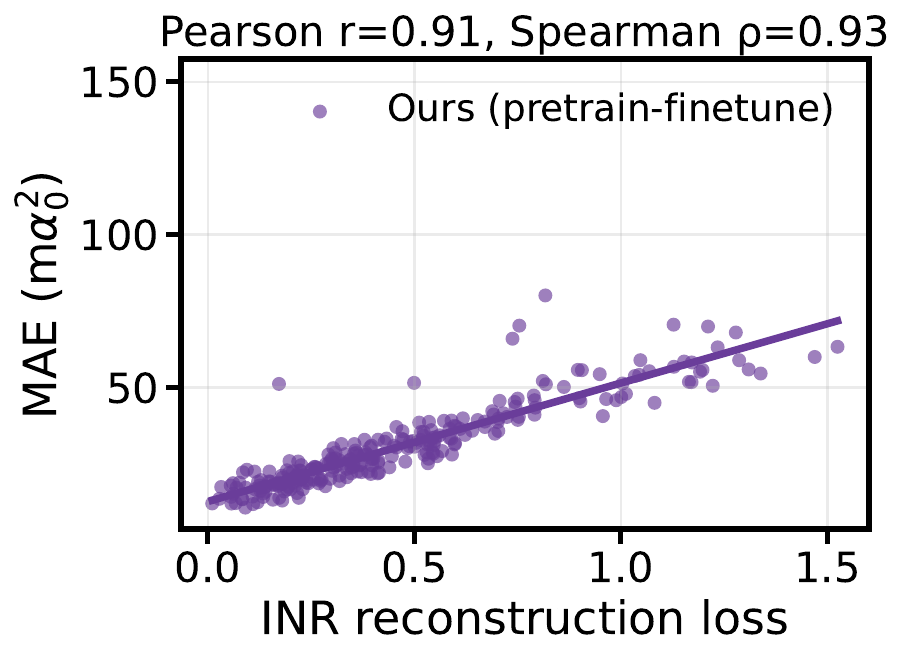}}
    \subfigure{\includegraphics[width=0.19\linewidth]{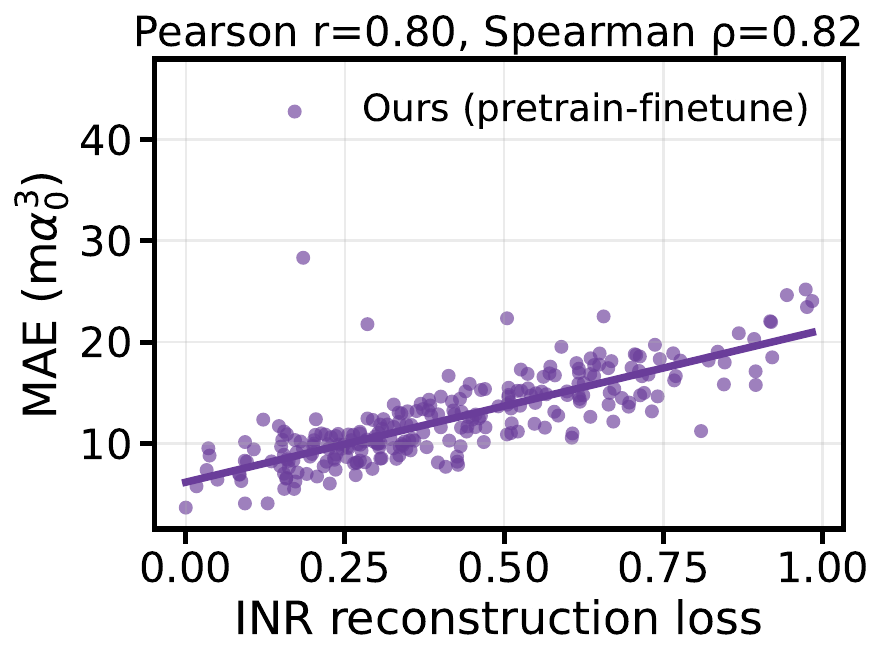}}
    \subfigure{\includegraphics[width=0.19\linewidth]{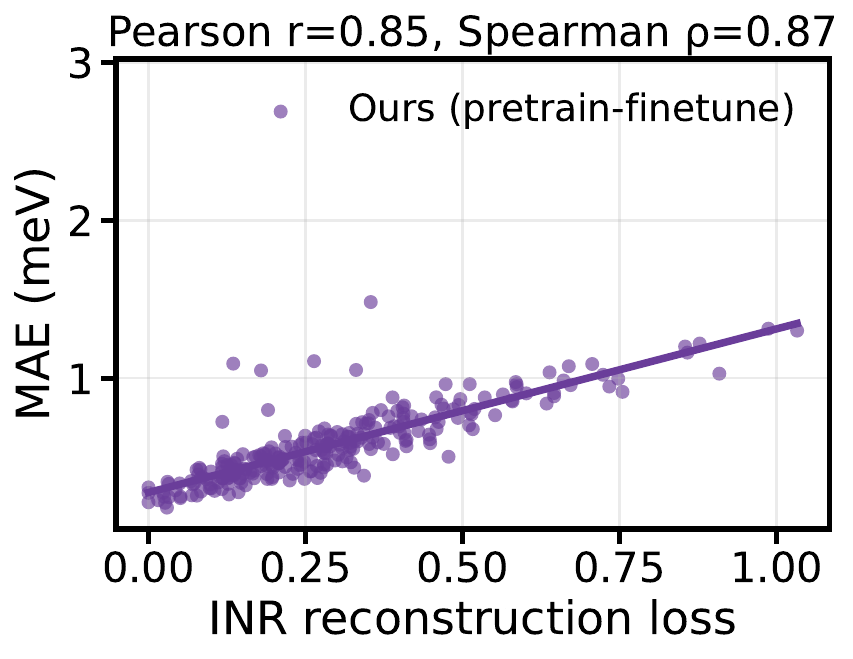}}

    \caption{
        \textbf{Additional correlation results between INR reconstruction loss and property prediction error (MAE) for MolField.}
        Each column corresponds to a molecular property, with the top row showing models trained from scratch and the bottom row showing models with pretraining followed by finetuning.
        Each point denotes one molecule.
        Pearson and Spearman correlation coefficients are reported in each subfigure.
        Consistent positive correlations across properties indicate that more accurate function-space reconstructions are associated with improved property prediction, and that pretraining further reduces both reconstruction loss and prediction error.
    }
    \label{fig:pp_dist_full}
\end{figure*}

\subsection{Full Distribution Analysis}\label{sec:full_distribution}

Beyond aggregate metrics, we analyze the relationship between function-space reconstruction quality and downstream prediction performance at the level of the full data distribution.
Figure~\ref{fig:pp_dist_full} visualizes the correlation between the INR reconstruction loss and property prediction error (MAE) for multiple molecular properties, under both training-from-scratch and pretrain-finetune settings.
Across all properties, we observe strong positive correlations between reconstruction loss and prediction error, as quantified by both Pearson and Spearman coefficients. These indicate that improvements in function reconstruction accuracy translate continuously into lower prediction error.

Pretraining the INR further strengthens this relationship.
As shown in Figure~\ref{fig:pp_dist_full} (bottom row), the pretrain-finetune regime shifts the joint distribution toward lower reconstruction loss and lower prediction error, while preserving the overall monotonic trend.
This behavior indicates that pretraining improves the quality of the learned function family rather than altering the mechanism by which downstream tasks are solved..

More broadly, this distribution-level dependence suggests that downstream prediction difficulty is governed by the intrinsic complexity of the underlying molecular function.
As the implicit representation more faithfully captures the continuous structure of a molecule, multiple physical properties become simultaneously easier to predict.
These results provide empirical evidence that \method learns structural regularities intrinsic to molecular systems, rather than relying on task-specific shortcuts or property-dependent heuristics.

\begin{figure}[!b]
    \centering
    \subfigure{\includegraphics[width=0.33\linewidth]{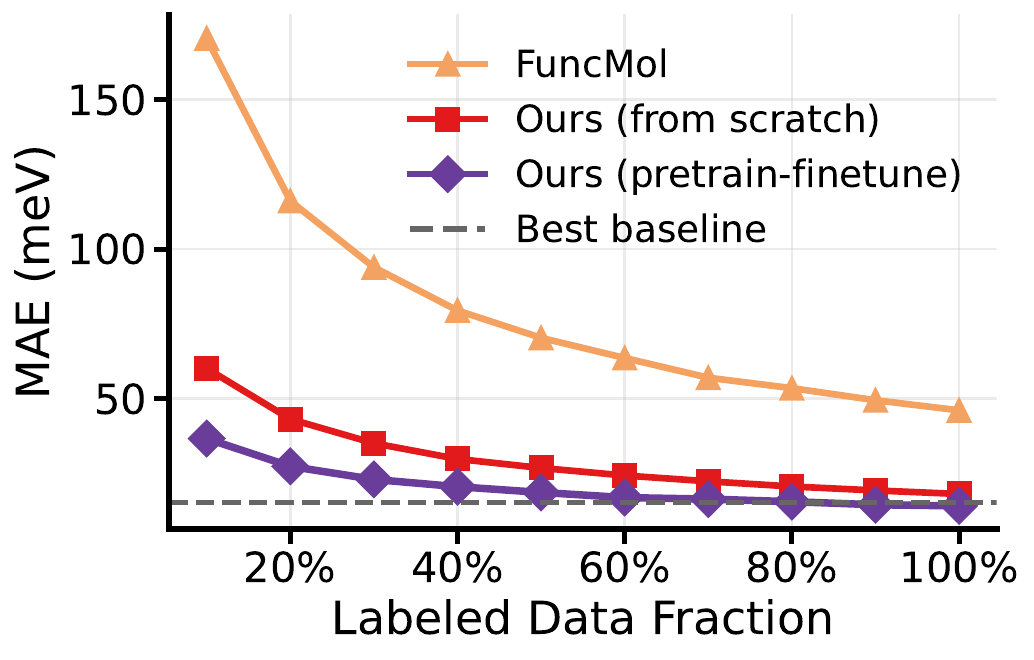}}
    \subfigure{\includegraphics[width=0.33\linewidth]{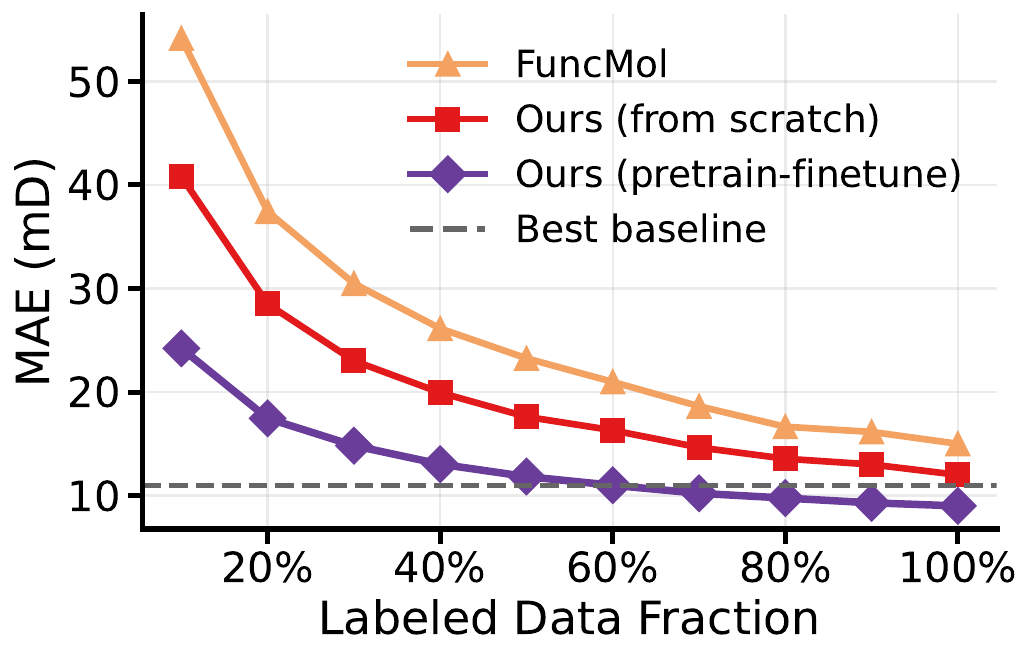}}
    \subfigure{\includegraphics[width=0.33\linewidth]{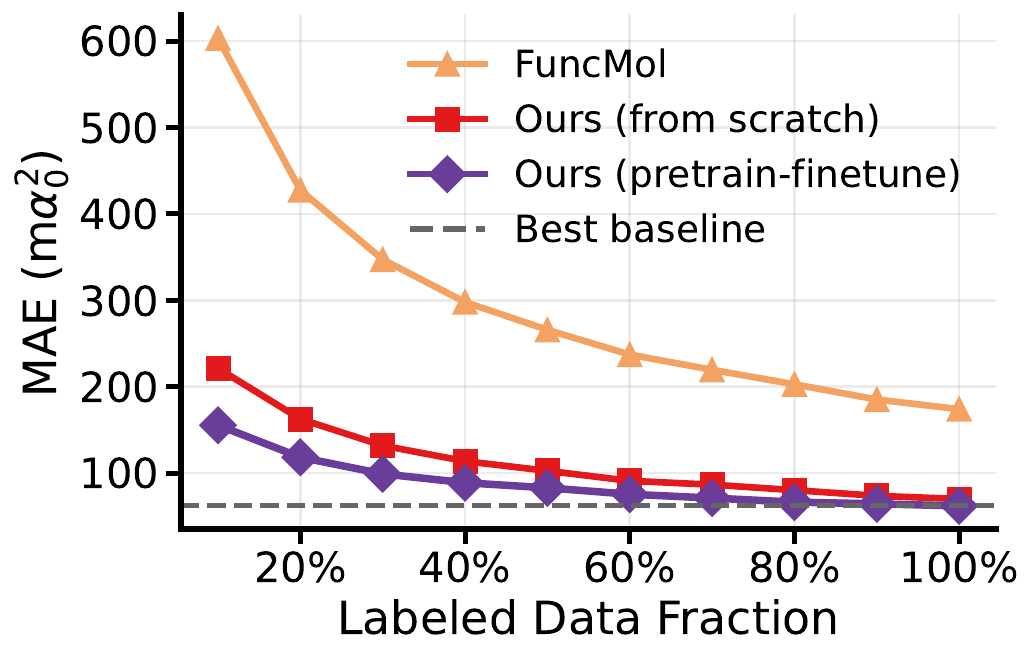}} \\
    \subfigure{\includegraphics[width=0.33\linewidth]{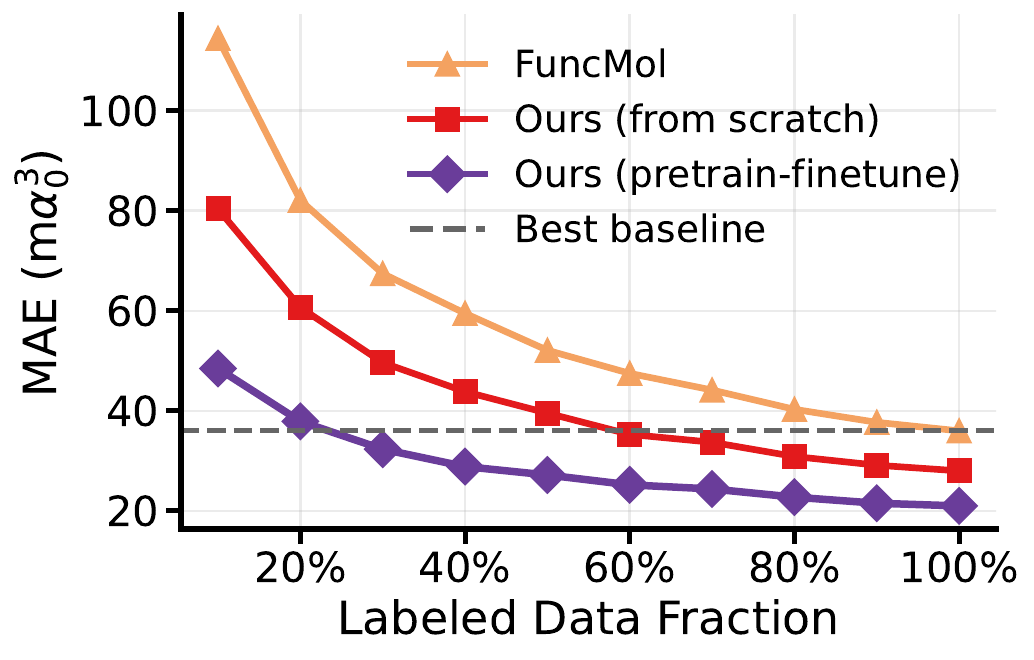}}
    \subfigure{\includegraphics[width=0.33\linewidth]{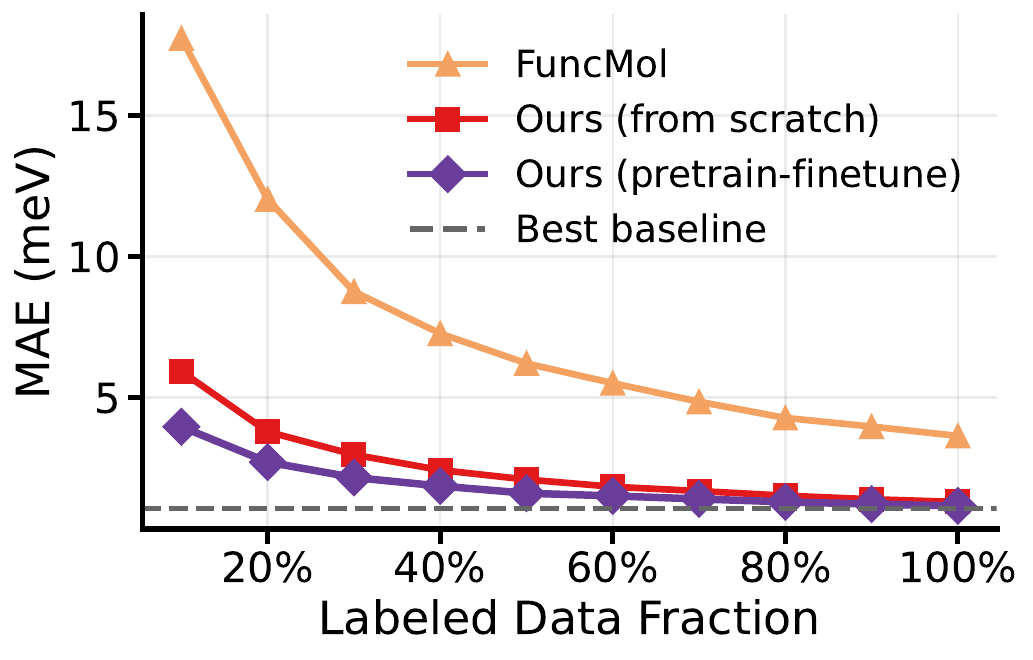}}

    \caption{\textbf{Property prediction results with different labeled data ratio.}
        \method consistently outperforms the primary baseline FuncMol across all label ratios and properties. Pretraining followed by finetuning further improves sample efficiency, especially in low-label regimes. The dashed line indicates the best-performing baseline with full data among prior methods for each property.}
    \label{fig:low_data_pp}
\end{figure}
\subsection{Low data analysis}\label{sec:low_data_app}
We evaluate property prediction under varying labeled data fractions to assess data efficiency and robustness.
As shown in Figure~\ref{fig:low_data_pp}, \method consistently achieves lower MAE than the primary baseline FuncMol across all labeled-data ratios and all evaluated molecular properties.
This difference can be attributed to the distinct inductive biases of the two approaches.
While FuncMol is designed to model instance-specific molecular configurations and excels at generation, molecular property prediction primarily depends on capturing structural regularities that are shared across molecules.
Rather than requiring exact reconstruction of individual atomic configurations, property prediction benefits from representations that emphasize common geometric and physical patterns, such as chemical bonding structures and relational atomic arrangements.
By learning molecular representations in function space via a hyper-network, \method prioritizes these shared structural characteristics, leading to consistently improved prediction accuracy.

The dashed lines in Figure~\ref{fig:low_data_pp} indicate the best-performing prior baseline for each property.
Compared against this reference, \method demonstrates competitive or superior data efficiency on several properties.
For example, on the property polarizability ($m\alpha_0^3$), \method achieves performance comparable to the best baseline using only $20\%$ of the labeled data, highlighting its strong sample efficiency in low-label regimes.

Figure~\ref{fig:low_data_pp} further compares two training strategies of \method: training from scratch and pretrain-finetune.
Across all label ratios, pretrain-finetune consistently outperforms training from scratch, with the largest gains appearing in the low-label regime.
Since pretraining is performed without access to target property labels, these improvements suggest that the model learns objective and transferable regularities of molecular structure in function space, which can be effectively reused to reduce the amount of labeled supervision required for downstream prediction.

Taken together, these results indicate that \method learns more than a task-specific predictor.
By capturing structural patterns intrinsic to molecular systems, it yields improved predictive performance and stronger sample efficiency, particularly in low-label regimes.

\end{document}